\documentclass{article}

 \usepackage[preprint]{neurips_2026}

\usepackage[utf8]{inputenc} 
\usepackage[T1]{fontenc}    
\usepackage[colorlinks=true, linkcolor=blue!70!black, citecolor=blue!60!black, urlcolor=blue!60!black]{hyperref}       
\usepackage{url}            
\usepackage{booktabs}       
\usepackage{amsfonts}       
\usepackage{nicefrac}       
\usepackage{microtype}      
\usepackage{xcolor}         
\usepackage{multirow}
\usepackage{graphicx}
\usepackage{amsmath}
\usepackage{amsthm}
\usepackage{cleveref}
\usepackage[most]{tcolorbox}
\usepackage{subcaption}  
\usepackage{tikz}
\usetikzlibrary{arrows.meta, positioning, calc, patterns, decorations.pathreplacing}

\newtheorem{proposition}{Proposition}
\newtheorem{remark}{Remark}

\newcommand{\KL}{\mathrm{KL}}
\newcommand{\E}{\mathbb{E}}
\newcommand{\given}{\,|\,}
\newcommand{\pto}{\propto}

\usepackage[scaled=0.9]{inconsolata}
\crefname{figure}{Fig.}{Figs.}

\tcbset{
  takeawaybox/.style={
    colback=purple!5,
    colframe=purple!75!black,
    boxrule=1pt,
    arc=4pt,
    left=6pt,
    right=6pt,
    top=6pt,
    bottom=6pt,
    enhanced
  }
}
\definecolor{m_green}{HTML}{2C8915}

\title{On-Policy Self-Distillation with Sampled Demonstrations Reduces Output Diversity}

\author{
  Andrei Liviu Nicolicioiu\textsuperscript{1,2,3,\dag}
  \And
  Mohammad Pezeshki\textsuperscript{3,*} 
  \And
  Aaron Courville\textsuperscript{1,2,4,*}
}

\begin{document}

\maketitle

\renewcommand{\thefootnote}{}
\footnotetext{%
\hspace{-0.5cm}\textsuperscript{1}\,Mila \quad
\textsuperscript{2}\,Universit\'e de Montr\'eal \quad
\textsuperscript{3}\,FAIR at Meta \quad
\textsuperscript{4}\,CIFAR AI Chair \\[2pt]
\textsuperscript{\dag}\,Correspondence: \texttt{andrei.nicolicioiu@mila.quebec} \quad
\textsuperscript{*}\,Equal advising
}
\renewcommand{\thefootnote}{\arabic{footnote}}

\begin{abstract}
On-policy self-distillation achieves strong pass@1 accuracy by using a single model as both teacher and student, with the teacher conditioned on a correct demonstration to provide dense token-level feedback. We show that this could come at a hidden cost: rollout diversity decreases and pass@k curves flatten (i.e., generating more rollouts fails to improve accuracy). We trace this to compounding biases in the design of self-distillation with sampled demonstrations. The teacher scores each student rollout while conditioned on a sampled correct rollout, channeling its feedback through the model's own biases. We theoretically analyze the optimal self-distillation policy and show that it tilts the base distribution by a pointwise conditional mutual information score between the student's rollout and the correct rollout used as context. Unlike the ideal optimal on-policy reinforcement learning (RL), which preserves probability ratios among equally correct rollouts, self-distillation can amplify existing probability gaps, concentrating mass on already-dominant modes. On a controlled graph path-finding task and science question-answering benchmarks, self-distilled models match or exceed RL on average performance but exhibit substantially lower \textit{functional} and \textit{semantic} diversity, failing on out-of-distribution settings that require diverse strategies.

\end{abstract}

\section{Introduction}

Current LLM post-training approaches to instill capabilities in models have different tradeoffs, with supervised fine-tuning (SFT) learning initial behaviors 
and on-policy RL methods refining and exploring new approaches \citep{zhang2025interplay_pre_mid_RL}.
In between, on-policy distillation \citep{rishab_on_policy_distill, onpolicydistillation_thinking_machines} uses a stronger teacher to guide a student using student-generated data.
Self-distillation goes further by eliminating the external teacher entirely: the same model, conditioned
on privileged information, such as a correct solution or environmental feedback, provides dense token-level feedback on the student's own generations.
Recent methods, including SDPO~\citep{sdpo}, SDFT~\citep{sdft_continual},  OPSD~\citep{opsd, penaloza2026privileged}, and OPCD~\citep{opcd}, instantiate this approach, achieving strong performance across several tasks such as scientific question-answering, continual learning tasks, and agentic tasks.

In this paper, we investigate a specific case, Self-Distillation with Sampled Demonstrations (SDSD) where student rollouts are guided by a teacher with demonstrations in its context. 
The demonstrations could come from correct student rollouts (exactly the setup of \citet{sdpo} with student demonstrations) or from external models \citep{opsd, penaloza2026privileged}.
Here, 
we find that good accuracy might come at a hidden cost. SDSD models with sampled demonstrations exhibit pass@k curves with small or nearly flat slopes (
Figures~\ref{fig:graph_results},~\ref{fig:sdpo_main_pass_k}): generating more rollouts fails to solve new problems. By contrast, models trained with on-policy RL (e.g., GRPO) show steep pass@k improvement, where each additional sample meaningfully increases problem coverage. SDSD could thus trade rollout diversity for average accuracy.

We argue that SDSD introduces compounding sources of bias that can progressively reduce rollout diversity. Intuitively, a rollout is \emph{aligned} with a demonstration when the two share more structural or stylistic features, causing the teacher conditioned on that demonstration to assign it higher probability and therefore reinforce it more strongly during training. This creates a bias toward solutions that resemble the sampled demonstrations. In particular, a teacher may struggle to effectively guide a correct but less typical rollout when conditioned on a more standard or canonical demonstration, simply because the two trajectories share fewer common patterns. As a result, distinctive yet valid solution strategies receive weaker learning signals. 
Over repeated training updates, this preference can compound: rollouts that are more aligned with previously sampled demonstrations become increasingly reinforced, while less aligned, but still correct, solutions are gradually suppressed (See Fig. \ref{fig:main_figure}). We hypothesize that this feedback loop contributes to the reduced rollout diversity and flattened pass@\(k\) scaling observed in SDSD models.

We formalize this effect by deriving the optimal self-distillation policy (Proposition~\ref{prop:main-sdkl-optimum}). The resulting policy is a tilted version of the base distribution, where the tilt is determined by the expected \emph{pointwise conditional mutual information} (PCMI) between a student rollout and the sampled demonstration. 
PCMI measures how much conditioning on a demonstration increases the model's preference for a particular rollout. Unlike a binary reward, which treats all correct rollouts equally, PCMI distinguishes among equally valid solutions and assigns greater weight to rollouts that are already more compatible with the demonstrations and the base policy. Consequently, self-distillation with sampled demonstrations can amplify existing probability imbalances: likely rollouts become increasingly likely, while less common but correct solutions are progressively suppressed, reducing rollout diversity.

To diagnose the reduced diversity, we use two notions of diversity beyond the commonly-used token-level entropy. \emph{Functional diversity} is the rate at which additional samples solve new problems, reflected in the slope of pass@k curves. \emph{Semantic diversity} measures whether rollouts differ in their high-level strategy (e.g., different paths through a graph, different proof approaches in math) rather than just surface-level wording. We show that token-level entropy fails to capture either notion (\S\ref{sec:entropy}).

\begin{figure}[t]
\centering
\begin{tikzpicture}[>=Stealth, every node/.style={font=\footnotesize}]

\definecolor{cOK}{HTML}{2E7D32}
\definecolor{cBad}{HTML}{999999}
\definecolor{cBlue}{HTML}{1565C0}
\definecolor{cOrange}{HTML}{D4760A}
\definecolor{cRed}{HTML}{C62828}

\def\colA{0}       
\def\colB{4.8}     
\def\colC{9.6}     
\def\colW{4.3}     


\begin{scope}[xshift=\colA cm]

\node[font=\footnotesize\bfseries] at (2.15, 0) {Regular RL (e.g., GRPO)};
\draw[black!40, thick] (0.0, -0.18) -- (\colW, -0.18);

\node[anchor=west] at (0.1, -0.55) {Question $x$ $\to$ Student generates:};

\node at (1.85, -1.0) {$y^1$};
\node at (2.4, -1.0)  {$y^2$};
\node at (2.95, -1.0) {$y^3$};
\node[cBad] at (3.5, -1.0) {$y^4$};

\node[cOK] at (1.85, -1.35) {\checkmark};
\node[cOK] at (2.4, -1.35)  {\checkmark};
\node[cOK] at (2.95, -1.35) {\checkmark};
\node[cBad] at (3.5, -1.35) {$\times$};

\foreach \x in {1.85, 2.4, 2.95, 3.5} {
    \draw[->, black!40] (\x, -1.5) -- (\x, -1.8);
}


\node at (1.85, -2.05) {$1$};
\node at (2.4, -2.05)  {$1$};
\node at (2.95, -2.05) {$1$};
\node[cBad] at (3.5, -2.05) {$0$};

\node[anchor=west, font=\scriptsize, black!55] at (0.1, -2.45) {Verifier: same reward for all correct};

\draw[black!15, densely dotted] (0.0, -2.75) -- (\colW, -2.75);

\node[anchor=west, font=\scriptsize] at (0.1, -3.05) {scores every correct $y^i$:};

\node[anchor=west, font=\scriptsize] at (0.1, -3.4) {$y^1$:};
\fill[cOK!55, rounded corners=1pt] (0.6, -3.5) rectangle (2.1, -3.32);
\node[anchor=west, font=\scriptsize, cOK!70!black] at (2.2, -3.4) {equal};

\node[anchor=west, font=\scriptsize] at (0.1, -3.75) {$y^2$:};
\fill[cOK!55, rounded corners=1pt] (0.6, -3.85) rectangle (2.1, -3.67);
\node[anchor=west, font=\scriptsize, cOK!70!black] at (2.2, -3.75) {equal};

\node[anchor=west, font=\scriptsize] at (0.1, -4.1) {$y^3$:};
\fill[cOK!55, rounded corners=1pt] (0.6, -4.2) rectangle (2.1, -4.02);
\node[anchor=west, font=\scriptsize, cOK!70!black] at (2.2, -4.1) {equal};

\draw[black!15, densely dotted] (0.0, -4.45) -- (\colW, -4.45);

\node[anchor=west, font=\scriptsize] at (0.1, -4.75) {All correct reinforced \textbf{equally}};

\end{scope}


\begin{scope}[xshift=\colB cm]

\node[font=\footnotesize\bfseries] at (2.15, 0) {Self-Distillation (SDSD)};
\draw[black!40, thick] (0.0, -0.18) -- (\colW, -0.18);

\node[anchor=west] at (0.1, -0.55) {Question $x$ $\to$ Student generates:};

\node at (1.85, -1.0) {$y^1$};
\node at (2.4, -1.0)  {$y^2$};
\node at (2.95, -1.0) {$y^3$};
\node[cBad] at (3.5, -1.0) {$y^4$};

\node[cOK] at (1.85, -1.35) {\checkmark};
\node[cOK] at (2.4, -1.35)  {\checkmark};
\node[cOK] at (2.95, -1.35) {\checkmark};
\node[cBad] at (3.5, -1.35) {$\times$};

\draw[->, cBlue, thick] (1.85, -1.5) -- (1.85, -2.5);
\node[anchor=west, font=\scriptsize, cBlue!70!black] at (0.1, -1.75) {\textit{$y^1$ selected for}};
\node[anchor=west, font=\scriptsize, cBlue!70!black] at (0.1, -2.05) {\textit{teacher context}};

\node[draw, rounded corners=2pt, fill=cOrange!6, inner sep=3pt] (tbox) at (2.15, -2.85) {Teacher: $p(y_{t}^i \mid x, y_{<t}^i, y^1)$};

\node[anchor=west, font=\scriptsize] at (0.1, -3.55) {reinforces every rollout \textbf{differently}:};

\node[anchor=west, font=\scriptsize] at (0.1, -3.9) {$y^1$:};
\fill[cOrange!80, rounded corners=1pt] (0.6, -4.0) rectangle (2.4, -3.82);
\node[anchor=west, font=\scriptsize, cOrange!80!black] at (2.5, -3.9) {strong};

\node[anchor=west, font=\scriptsize] at (0.1, -4.25) {$y^2$:};
\fill[cOrange!50, rounded corners=1pt] (0.6, -4.35) rectangle (1.5, -4.17);
\node[anchor=west, font=\scriptsize, cOrange!55!black] at (1.6, -4.25) {some};

\node[anchor=west, font=\scriptsize] at (0.1, -4.6) {$y^3$:};
\fill[cOrange!22, rounded corners=1pt] (0.6, -4.7) rectangle (0.85, -4.52);
\node[anchor=west, font=\scriptsize, black!40] at (0.95, -4.6) {little};

\end{scope}

\begin{scope}[xshift=\colC cm]

\node[font=\footnotesize\bfseries] at (2.15, 0) {Functional Diversity: pass@[$1 \rightarrow k$]};
\draw[black!40, thick] (0.0, -0.18) -- (\colW, -0.18);

\node[anchor=north west, inner sep=0] at (0.0, -0.35) {%
    \includegraphics[width=\colW cm]{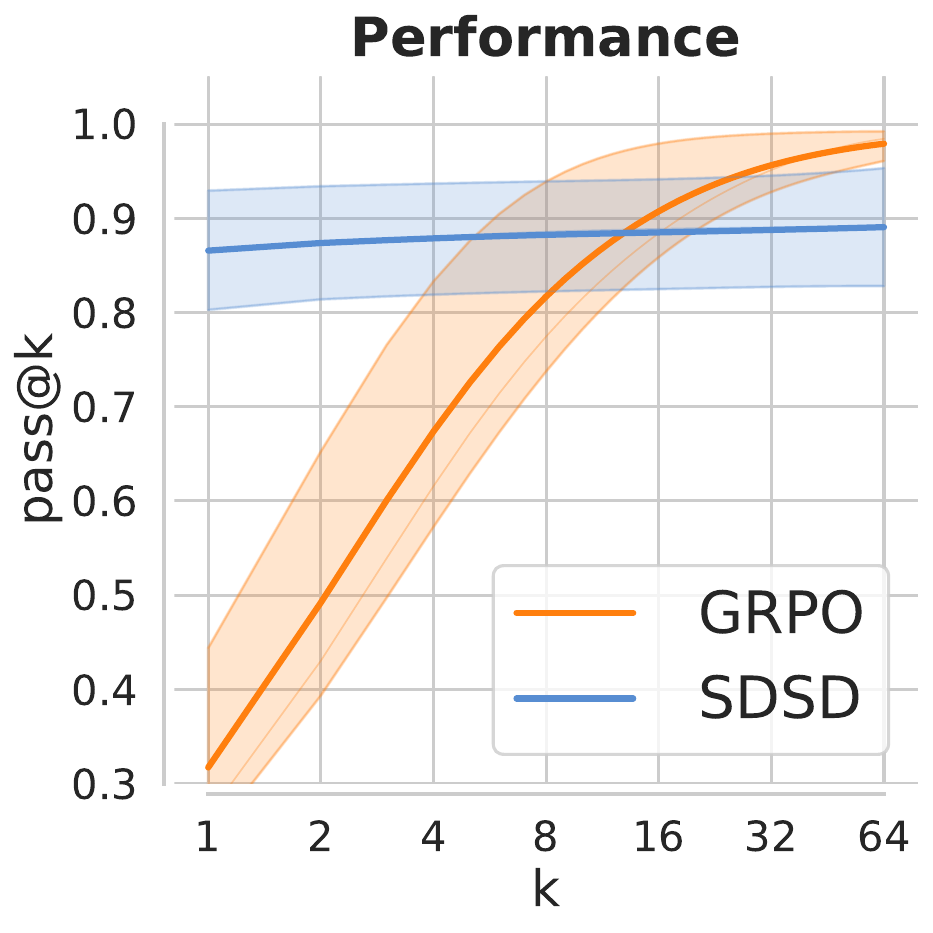}%
};

\end{scope}

\end{tikzpicture}
\caption{
RL and self-distillation treat correct rollouts differently, with consequences for rollout diversity.
\textbf{Left:} A binary verifier gives equal reward to all correct rollouts, so RL reinforces them uniformly.
\textbf{Middle:} Self-distillation conditions the teacher on a sampled correct rollout, typically the most probable one, so the teacher's feedback is strongest for similar rollouts and weakest for those taking a different approach.
\textbf{Right:} 
On a Graph Path finding task,
this gap manifests as flatter pass@$k$ curves for  Self-Distillation with Sampled Demonstrations (SDSD): generating more rollouts fails to solve new problems, unlike GRPO where each additional sample meaningfully increases coverage.
}
\label{fig:main_figure}
\end{figure}
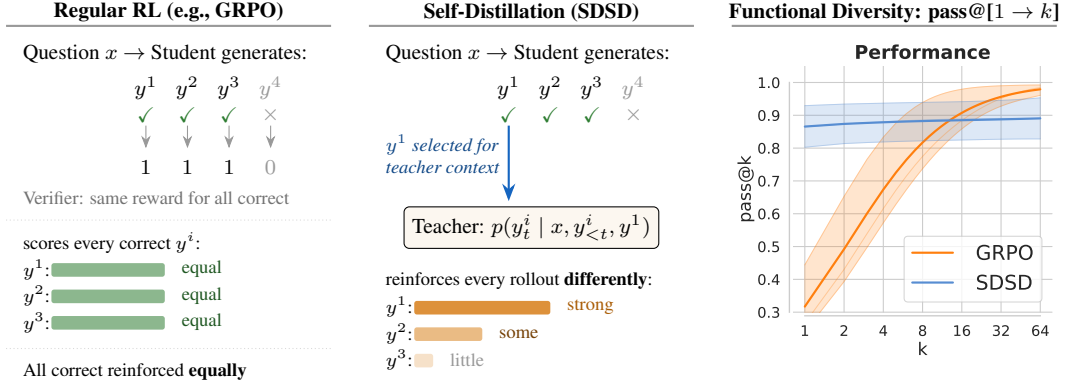

Our contributions are:
\begin{itemize}
    \item We prove that the optimal policy of self-distillation with sampled demonstrations tilts the base distribution by expected PCMI rather than reward, and that this can amplify probability gaps among equally correct rollouts, a property absent from standard RL (Prop.~\ref{prop:main-sdkl-optimum}, Remark~\ref{rem:main-sdkl-ratio}).
    \item We introduce a graph path-finding task, in which \textit{semantic} diversity, the number of distinct concept categories explored, is precisely measurable and directly predicts out-of-distribution generalization (\S\ref{sec:concept-graph}, Fig.~\ref{fig:graph_results}). 
    \item On graph path-finding and science QA tasks \citep{sciknoweval}, we show that self-distillation with sampled demonstrations achieves competitive pass@1 but substantially lower \textit{functional} and \textit{semantic} diversity than RL, failing on out-of-distribution tasks that require diverse strategies (
    Fig.~\ref{fig:graph_results},~\ref{fig:sdpo_main_pass_k}).
\end{itemize}

\section{Background: Self-Distillation with Sampled Demonstrations}
\label{sec:self-distillation}

We review the self-distillation with sampled demonstration framework of~\citet{sdpo}, which uses a correctly verified rollout as privileged context for the teacher.
For each question $x$, the student policy generates a group of $N$ rollouts $y_n \sim \pi_\theta(\cdot \mid x)
$,
$
\mathcal{Y}(x) = \{y_1, \dots, y_N\}$.

Let $\mathrm{C}(x) \subseteq \mathcal{Y}(x)$ denote the subset verified as correct.
For each rollout $y \in \mathcal{Y}(x)$, a correct rollout $y^{\mathrm{corr}} \in \mathrm{C}(x)$ is sampled uniformly from this subset and provided to the teacher as context.

The teacher is an exponential moving average (EMA) of the student with parameters $\bar{\theta}$, conditioned on the question $x$ and the correct rollout $y^{\mathrm{corr}}$:
$\pi_{\bar{\theta}}(\cdot \mid x, y^{\mathrm{corr}}, y_{<t})$.
The training objective minimizes the token-level KL divergence between the student and this context-conditioned teacher:
\[
\mathcal{L}_{\mathrm{SD}}(\theta; x)
=
\frac{1}{N}
\sum_{y \in \mathcal{Y}(x)}
\mathbb{E}_{y^{\mathrm{corr}} \sim \mathrm{C}(x)}
\left[
\sum_{t=1}^{|y|}
\mathrm{KL}\!\left(
\pi_\theta(\cdot \mid x, y_{<t})
\,\middle\|\,
\mathrm{sg}\!\left[\pi_{\bar{\theta}}(\cdot \mid x, y^{\mathrm{corr}}, y_{<t})\right]
\right)
\right],
\]
where $\mathrm{sg}[\cdot]$ denotes stop-gradient. The corresponding token-level gradient~\citep{sdpo} is:
\[
\nabla_\theta \mathcal{L}_{\mathrm{SD}}(\theta; x)
=
\frac{1}{N}
\sum_{y \in \mathcal{Y}(x)}
\mathbb{E}_{y^{\mathrm{corr}}}
\left[
\sum_{t=1}^{|y|}
\mathbb{E}_{\hat{y}_t \sim \pi_\theta}
\left[
\log
\frac{
\pi_\theta(\hat{y}_t \mid x, y_{<t})
}{
\pi_{\bar{\theta}}(\hat{y}_t \mid x, y^{\mathrm{corr}}, y_{<t})
}
\;
\nabla_\theta \log \pi_\theta(\hat{y}_t \mid x, y_{<t})
\right]
\right].
\]

Self-distillation thus replaces a single scalar reward for the full sequence with a dense per-token correction signal derived from a context-conditioned version of the model itself.

\paragraph{On-policy RL is mode-seeking, SDSD is even more.}
Any on-policy method exhibits mode-seeking behavior concentrating its rollouts into a smaller subset \citep{llm_on_policy, on_policy_forgetting}. SDSD as a reverse KL objective~\citep{bishop2006pattern} has the same behavior
but even more pronounced due to \textit{additional compounding biases}.
In our setting, self-distillation involves two crucial samplings: that of the student rollout and that of the demonstration (either from the student's correct rollouts or from an external source).
The
\textit{alignment between these two} introduces a bias towards common responses. 
Consider how two student rollouts with equal reward, one common and one highly novel, are treated. It is more likely to sample a demonstration that resembles the common rollout. Then, the teacher has a bias to give more probability to rollouts that are similar to its context, meaning similar to the demonstration. Together this will lead to the common rollout being upweighted more than the unique one.
This leads to a rich-get-richer (likely-get-likelier)
behavior stronger than in standard RL, arising from both double sampling and its alignment and the preferences of the teacher.

Moreover, while all on-policy methods exhibit mode-seeking behavior due to \textit{optimization}, self-distillation uniquely introduces incentives for loss of diversity even at the level of the \textit{optimal policy.}

\section{Optimal Policy of Self-Distillation}
\label{sec:theory}
We derive the optimal self-distillation policy and characterize how it differs from standard RL. For ease of presentation, we use a sequence level objective first, and refer to the Appendix~\ref{app:detailed-derivations} for full derivations and token-level presentation \cref{app:token_level}. 
 Let $x$ denote the input prompt, $y$ an output sequence, and $\pi_0(y \given x)$ the base policy. We optimize a student policy $\pi(y \given x)$. All distributions are assumed strictly positive on their support.

\begin{proposition}[Optimal policy for standard KL-regularized RL]
\label{prop:main-rl-optimum}
The standard RL objective is:
\begin{equation}
\label{eq:main-rl-objective}
\max_{\pi} \; \E_{y \sim \pi(\cdot \given x)}
\left[
R(y \given x)
\right]
- \beta_{\mathrm{RL}} \, \KL\!\left(\pi(\cdot \given x) \,\|\, \pi_0(\cdot \given x)\right).
\end{equation}
It is well known~\citep{korbak2022rl_kl, rafailov2023_dpo} that the optimal policy of this objective is the following tilted distribution:
\begin{equation}
\label{eq:main-rl-optimum}
\pi^*_{\mathrm{RL}}(y \given x)
\pto
\pi_0(y \given x)
\exp\!\left( \frac{1}{\beta_{\mathrm{RL}}} R(y \given x) \right).
\end{equation}
\end{proposition}

The optimal RL policy consists of the base policy modulated by the reward. Thus, two rollouts with the same probability under the base and the same reward will be \textit{as likely} under the optimal policy.

We now derive the analogous result for self-distillation. We take the teacher to be the fixed base policy conditioned on a correct demonstration, and optimize the student to minimize the reverse KL with the teacher, regularized by a KL penalty to the base model.

In practice, explicit KL regularization to the base policy is not always used in RL or self-distillation. However, since training starts from the base policy and models are rarely trained to convergence, the 
policy typically remains close in KL to its starting point, making this a reasonable modeling choice.

Let $p_{\mathrm{corr}}(\cdot \given x)$ denote a reference distribution over correct demonstrations for input $x$, from which $y^{\mathrm{corr}}$ is sampled. In practice, this can be the empirical distribution over $\mathrm{C}(x)$, the student's own correct rollouts (\S\ref{sec:self-distillation}), or a distribution over external demonstrations (\S\ref{sec:demos_diversity}).

Let us define the following objective that takes the expectation over demonstrations.

\begin{proposition}[Optimal policy for SDSD-KL]
\label{prop:main-sdkl-optimum}
Let the self-distillation + KL objective:
\begin{equation}
\label{eq:main-sdkl-objective}
\min_{\pi} \; \E_{y^{\mathrm{corr}} \sim p_{\mathrm{corr}}(\cdot \given x)}
\Big[
\KL\!\left(\pi(\cdot \given x) \,\|\, 
\pi_0(y \given x, y^{\mathrm{corr}})
\right)
\Big]
+ \beta \, \KL\!\left(\pi(\cdot \given x) \,\|\, \pi_0(\cdot \given x)\right).
\end{equation}
The optimal policy of \eqref{eq:main-sdkl-objective} is

\begin{equation}
\label{eq:main-sdkl-optimum}
\pi^*_{\mathrm{SD\text{-}KL}}(y \given x)
\pto
\pi_0(y \given x)
\exp\!\left(
\frac{1}{1+\beta}
\, \E_{y^{\mathrm{corr}} \sim p_{\mathrm{corr}}(\cdot \given x)}
\big[
 i(y; y^{\mathrm{corr}} \mid x)
\big]
\right),
\end{equation} where the pointwise conditional mutual information (PCMI) is defined as:
\begin{equation}
\label{eq:main-pcmi}
i(y; y^{\mathrm{corr}} \mid x)
:=
\log \frac{\pi_0(y \given x, y^{\mathrm{corr}})}{\pi_0(y \given x)}.
\end{equation}
\end{proposition}

The RL optimal policy~\eqref{eq:main-rl-optimum} tilts the base policy by the reward. The self-distillation optimum instead tilts by the expected PCMI~\eqref{eq:main-pcmi}: a log-ratio measuring how much more likely the teacher finds $y$ after conditioning on a demonstration. When $y^{\mathrm{corr}}$ is relevant and supports $y$, the PCMI is positive; when $y^{\mathrm{corr}}$ is contradictory to $y$, it is negative. The base policy is thus tilted not by task reward, but by the teacher's assessment of how well each rollout aligns with correct demonstrations.

\begin{remark}[Ratio for two correct sequences under SDSD-KL]
\label{rem:main-sdkl-ratio}
Let $y_1$ and $y_2$ be two correct sequences for the same input $x$, and suppose $\pi_0(y_1 \given x)=k \, \pi_0(y_2 \given x) $
for some $k\geq1$. Then
\begin{equation}
\label{eq:main-sdkl-ratio}
\frac{\pi^*_{\mathrm{SD\text{-}KL}}(y_1 \given x)}{\pi^*_{\mathrm{SD\text{-}KL}}(y_2 \given x)}
=
 k \, \exp\!\left(
\frac{1}{1+\beta}
\E_{y^{\mathrm{corr}} \sim p_{\mathrm{corr}}(\cdot \given x)}
\big[
 i(y_1; y^{\mathrm{corr}} \mid x)-i(y_2; y^{\mathrm{corr}} \mid x)
\big]
\right).
\end{equation}
\end{remark}
\paragraph{When does sharpening occur?} SDSD-KL preserves the base-policy ratio $k$ only when the two sequences have the same expected PCMI. If the demonstrations $y^\mathrm{corr}$ support on average $y_1$ more than $y_2$, the ratio between the two rollouts becomes even larger under self-distillation, leading to sharpening where likely rollouts become even more likely. This contrasts with the RL optimal policy~\eqref{eq:main-rl-optimum}, which maintains the initial ratio of $k$ when both rollouts are equally correct, since the reward tilt cancels out.
This shows sharpening occurs when the expected PCMI is higher for the already-probable rollout. 
Such rollouts are more likely to be aligned to the demonstrations, and the teacher shares the same biases as the student, thus this becomes highly likely. The same derivations and implications carry over to the token-level objective, resulting in a bias for sharpening the distribution of the next-token, leading to loss of diversity of the whole rollout (see Appendix \ref{app:token_level}).

All on-policy learning methods, like GRPO or self-distillation have mode-seeking behavior due to \textit{optimization}, but the previous remark shows that self-distillation has an \textit{optimal} policy that can be sharper than the initial one.

\begin{figure}[t]
    \centering
    \includegraphics[width=1.0\linewidth]{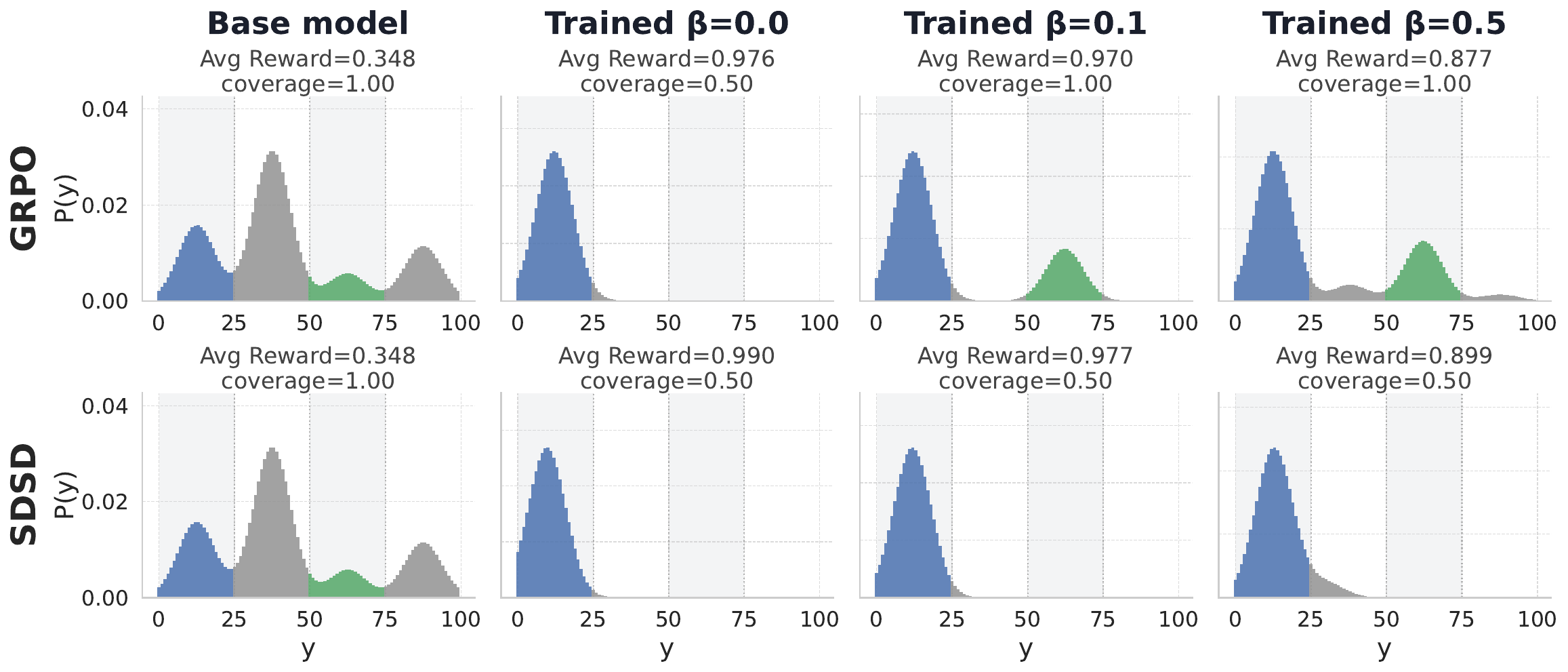}
    \caption{
    Illustrative example: SDSD collapses to a single high-reward mode, regardless of KL regularization,  while GRPO with KL regularization > 0 covers both modes. The environment has two equally-valued reward regions, so an ideal policy would maintain coverage of both. 
    }
    \label{fig:toy_results_theory}
\end{figure}

\subsection{Illustrative Example: Mode Collapse Under Self-Distillation}

We verify the sharpening predicted by our theory in a minimal controlled environment. The action space is $D{=}100$ discrete actions split into four quarters: the first and third quarters are rewarded, while the other two are not. An ideally diverse policy would place mass on both rewarded modes.

We parameterize the student policy $\pi_\theta$ over a four-bump base distribution. The teacher for SDSD is constructed as
$\pi_{\bar\theta}(y \mid y^{\mathrm{corr}}) \propto \pi_\theta(y) \cdot K(y, y^{\mathrm{corr}})$,
where $K$ is a Gaussian kernel centered on a correct sample $y^{\mathrm{corr}}$ drawn from the student's own correct outputs. This locally upweights the student's mass near each observed correct sample, exactly the mechanism that produces PCMI sharpening in our theory.

Figure~\ref{fig:toy_results_theory} shows the result. GRPO recovers both rewarded modes for any $\beta>0$. SDSD collapses to whichever rewarded region the base policy slightly favors and stays there for every $\beta$: the more probable region produces more correct samples, those samples become teacher contexts, and the teacher's kernel-shaped feedback reinforces nearby points, a self-reinforcing loop that no level of KL regularization to the base undoes.

\section{Experiments}

\subsection{Concept-Graph Setting: Loss of Semantic Diversity and OOD Performance}
\label{sec:concept-graph}

We design a controlled setting, challenging for LLMs, with a precise definition of \textit{semantic diversity} and a direct link between diversity and downstream performance.

\paragraph{Controlled experimental setting.}

We introduce a graph path-finding task, where we generate multiple graphs and create a query for each one. For each query, we prompt an LLM with a representation of the graph in context and ask it to generate a path between two points. See \cref{app:concept_graph} for an example of such query. A graph node represents an instance of a named concept (e.g., specific birds: heron, pigeon; fruits: orange, cherry) as seen in Fig. \ref{fig:concept_graph_example}. The graph has a star structure: a central start node connects to multiple \emph{concept chains}, each consisting of nodes from the same concept  (e.g., all birds or all fruits), each ending at a shape node (e.g., diamond, square). Two of the endpoints have the same name and represent the target; the remaining two are distractors. Multiple valid paths to the targets exist, each passing through a different concept chain. Training graphs are biased: some paths to the target are short ($11$ nodes, easier), while others are long ($15$ nodes, harder), creating an incentive for models to exploit easy routes and ignore harder but equally valid alternatives.
  
To check the robustness of the learned models, we evaluate on an \textit{in-distribution} test set and two out-of-distribution datasets: a \textit{larger-graphs} dataset, in which all paths to end nodes have a fixed length of 20, and a \textit{harder-graphs} dataset, in which all paths have fixed length of 11, but one edge is removed from chains leading to one of the two target nodes, giving fewer valid solutions.

\paragraph{Baselines.}

We train Qwen3-1.7B~\citep{yang2025qwen3} models using variants of GRPO and SDSD, on a dataset of 16k graphs for training and 128 for testing. We train for 1000 steps, with a mini-batch size of 16 queries, 4 rollouts per query, and a maximum generation length of 8,192 tokens. The experiments are trained on a single GPU using the library of \citet{NanoAhaMoment2025}.

\paragraph{GRPO.}
  Standard GRPO~\citep{shao2024_deepseekmath_grpo} serves as the primary baseline. At each iteration, the policy model generates N=4 rollouts per query. Within each group of N rollouts, a scalar reward is used to compute a scalar advantage for the whole sequence. 

\paragraph{GRPO+diversity.}
\looseness=-1
  GRPO+diversity adds a diversity reward to the score reward to encourage the model to explore different concept chains across its $N$ rollouts. 
  For each rollout, its diversity score is the fraction of the other $N-1$ rollouts in the same group that used a disjoint set of concepts. A rollout using different concepts than all the other rollouts of the same query, will have a diversity score of 1. 
  This relates to \citet{li2025_darling}, 
  who multiply the score reward by a clustering-based diversity score.

\paragraph{SDSD.} SDSD implementation following SDPO~\citep{sdpo},
where for each student rollout, the teacher is conditioned on another correct student rollout for the same query.

All models add to the main loss a KL regularization to the reference model.

The results in Fig. \ref{fig:graph_results} show that SDSD achieves good in-distribution pass@1 and the best pass@1 on the Larger Graph setting. However, its pass@k performance increases very slowly, or not at all for the harder settings. This flat pass@k curve indicates low \textit{functional diversity}: successive samples rarely solve new queries. The Harder Graph setting can only be solved by models that learned diverse rollouts during training. SDSD's failure there confirms that it relied exclusively on easy routes.

\begin{figure}[t]
    \centering
    \begin{subfigure}{0.58\textwidth}
        \centering
        \includegraphics[width=\linewidth]{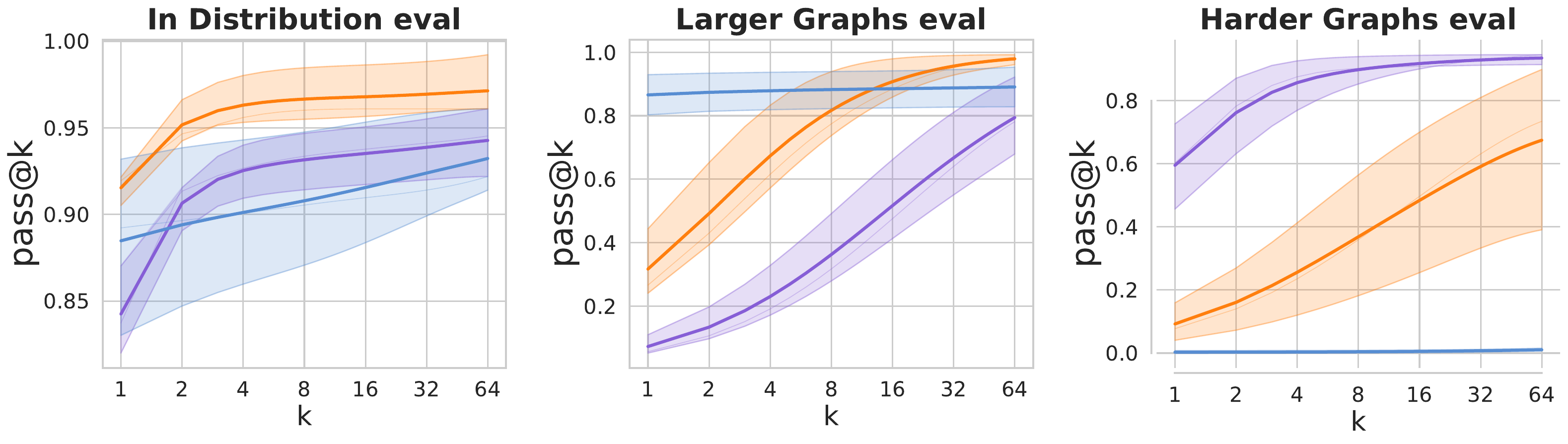}
        \label{fig:graph_results_passk}
    \end{subfigure}
    \hfill
    \begin{subfigure}{0.4\textwidth}
        \centering
        \includegraphics[width=\linewidth]{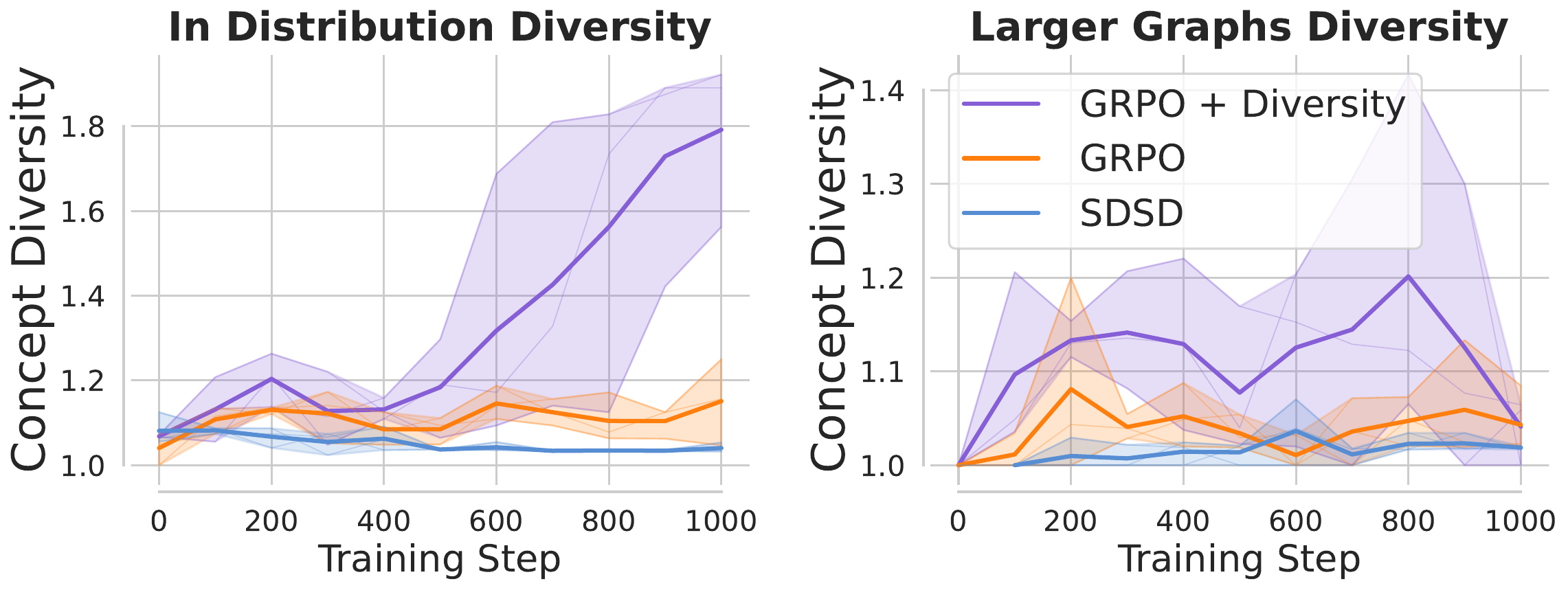}
        \label{fig:graph_results_diversity}
    \end{subfigure}
    \caption{
    In-distribution and Larger Graphs evaluations show SDSD has good pass@1 performance but, toward pass@k, the curve has a small slope highlighting low \textit{functional diversity}. 
    The third setup requires the model to have learned diverse rollouts during training, and is completely unsolved by self-distillation. Additionally, the last two figures show that the explicitly defined \textit{semantic diversity} of SDSD is the lowest. All runs train Qwen3-1.7B; mean and min/max runs with 3 seeds are shown.}
    \label{fig:graph_results}
\end{figure}

\paragraph{Semantic diversity.}
We seek rollouts with \textit{semantic diversity}, capturing meaningful variations in their trajectories, such as different high-level strategies or approaches.  In mathematical reasoning, for instance, one might want geometric vs.\ algebraic approaches, or different theorems. In the graph setting, we define the \emph{semantic diversity} of a set of rollouts as the number of unique concepts present across all of them.
This measures whether a model explores fundamentally different strategies (e.g., following animal chains vs.\ flower chains) rather than mere surface-level token variation.
For each query in the in-distribution and larger graphs testsets, we sample 64 rollouts and compute the average number of unique concepts of the nodes present in them, and show these diversity scores across training in Fig. \ref{fig:graph_results}. We find that GRPO has relatively low semantic diversity, but adding a diversity reward significantly improves it. On the other hand, SDSD has the lowest semantic diversity scores, which are correlated with the low functional diversity (slope of pass@k) and low scores in the Harder Graphs dataset that requires diversity.

\subsection{Science QA: Functional Diversity in Practice}
\label{sec:qa}

We highlight the interplay of accuracy and diversity of SDSD and GRPO in science QA settings. 
\paragraph{Setup.} We evaluate on four verifiable reasoning tasks spanning scientific knowledge drawn from SciKnowEval \citep{sciknoweval}, a benchmark of multiple-choice science questions.
 Across all tasks, we train on the training split and evaluate on the held-out test split by generating $N{=}16$ rollouts per question and reporting mean accuracy (pass@1) as well as pass@$k$.

 All experiments follow the implementation and configuration of \citet{sdpo}, and use Qwen3-8B~\citep{yang2025qwen3} and Olmo-3-7B-Instruct~\citep{olmo3} as the base models, trained with AdamW~\citep{adamW} for up to 30 epochs bounded by a 5h training time budget.
At each training step, the policy generates $N{=}8$ rollouts per question for a batch of 32 questions, sampled with temperature 1.0.
 Each configuration is run with 3 seeds on 4 Nvidia H200 GPUs.
 We compare the following models: 

\textbf{GRPO}. Standard GRPO generating $N{=}8$ rollouts per question for a batch of 32 questions. The batch is optimized in mini-batches of 8 questions.

\textbf{SDSD ($K=1$)}. SDPO-style self-distillation \citep{sdpo} where the teacher is the base model conditioned on one correct student rollout.

\textbf{SDSD ($K=3$)}. We introduce a baseline using an ensemble of $K=3$ teachers.
Everything is the same as above, but we collect three distinct correct demonstrations (if available) and create an independent teacher from each, then average their distillation losses per student rollout. This requires $K$ forward passes through the teacher, though it adds only $5.8\%$ wall-clock time per training step since generating the student rollouts dominates computation.

Fig. \ref{fig:sdpo_main_pass_k} shows SDSD variants have flatter pass@k curves than GRPO, indicating lower functional diversity. Table \ref{tab:sdpo_best_ckpt_perf_diversity_combined} reports mean accuracy across 3 seeds, using the best checkpoint per method by average accuracy. As shown in Fig. \ref{fig:qa_training_curves}, SDSD variants attain higher pass@1 but lower pass@k throughout training, consistent with reduced diversity.

\begin{figure}[t]
    \centering
    \includegraphics[width=1\linewidth]{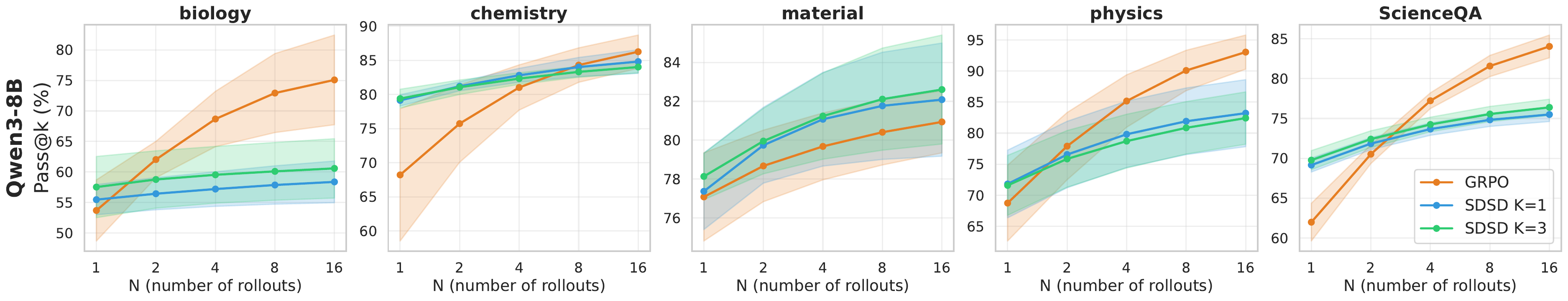}
    \includegraphics[width=1\linewidth]{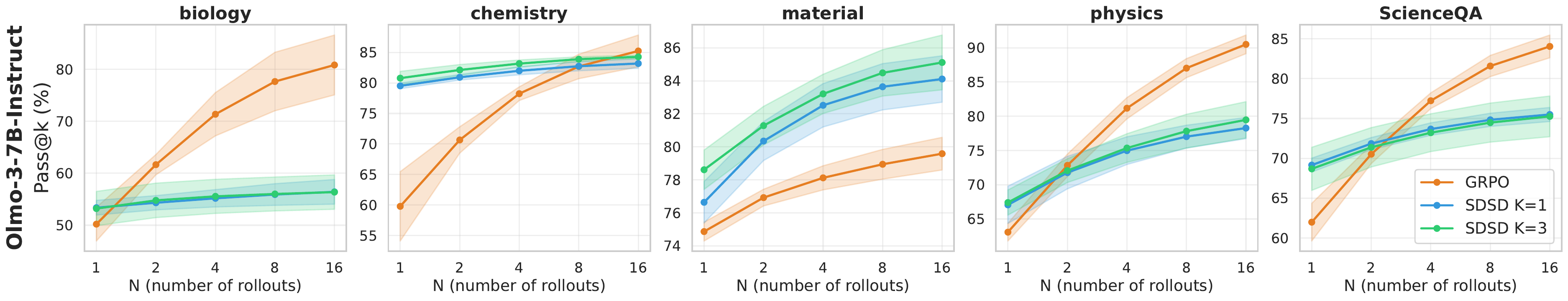}
    
    \caption{Pass@k curves for Science QA (4 tasks and average across tasks). SDSD achieves better pass@1 but its curves flatten quickly, indicating low functional diversity.
    mean $\pm$ stderr over 3 seeds.}
    \label{fig:sdpo_main_pass_k}
\end{figure}

\begin{table}[ht]
\centering
\caption{Pass@1 and Pass@16 at the best checkpoint per dataset, for Qwen3-8B and Olmo-3-7B-Instruct. Overall both SDSD variants have better Pass@1 but worse Pass@16, indicating low diversity between rollouts. Mean $\pm$ std across 3 seeds. \textbf{Bold} = best per column.}
\label{tab:sdpo_best_ckpt_perf_diversity_combined}
\resizebox{1.0\textwidth}{!}{%
\begin{tabular}{lcccccccccc}
\toprule
\multirow{2}{*}{\textbf{Method}} & \multicolumn{2}{c}{\textbf{Biology}} & \multicolumn{2}{c}{\textbf{Chemistry}} & \multicolumn{2}{c}{\textbf{Material}} & \multicolumn{2}{c}{\textbf{Physics}} & \multicolumn{2}{c}{\textbf{Average}} \\
\cmidrule(lr){2-3}\cmidrule(lr){4-5}\cmidrule(lr){6-7}\cmidrule(lr){8-9}\cmidrule(lr){10-11}
  & Pass@1 & Pass@16 & Pass@1 & Pass@16 & Pass@1 & Pass@16 & Pass@1 & Pass@16 & Pass@1 & Pass@16 \\
\midrule
\multicolumn{11}{l}{\textit{Qwen3-8B}} \\
\midrule
GRPO & $57.2_{\pm3.9}$ & $\mathbf{69.3_{\pm2.9}}$ & $76.6_{\pm3.6}$ & $\mathbf{87.3_{\pm1.7}}$ & $\mathbf{79.3_{\pm0.4}}$ & $\mathbf{82.3_{\pm0.7}}$ & $74.4_{\pm2.9}$ & $\mathbf{95.6_{\pm0.5}}$ & $71.9_{\pm2.0}$ &
$\mathbf{83.6_{\pm1.1}}$ \\
SDSD $K$=1     & $57.5_{\pm1.2}$ & $60.3_{\pm3.2}$ & $\mathbf{78.8_{\pm0.9}}$ & $86.5_{\pm0.1}$ & $78.1_{\pm2.3}$ & $80.1_{\pm2.6}$ & $\mathbf{76.6_{\pm2.6}}$ & $88.1_{\pm2.7}$ & $72.7_{\pm1.1}$ & $78.7_{\pm1.2}$ \\
SDSD $K$=3     & $\mathbf{61.8_{\pm1.5}}$ & $64.8_{\pm1.0}$ & $78.5_{\pm0.4}$ & $83.9_{\pm1.1}$ & $77.7_{\pm1.0}$ & $80.1_{\pm0.6}$ & $75.8_{\pm2.8}$ & $85.4_{\pm3.3}$ & $\mathbf{73.4_{\pm0.8}}$ & $78.5_{\pm0.6}$ \\
\midrule
\multicolumn{11}{l}{\textit{Olmo-3-7B-Instruct}} \\
\midrule
GRPO & $50.2_{\pm3.2}$ & $\mathbf{80.9_{\pm5.8}}$ & $59.8_{\pm5.7}$ & $\mathbf{85.2_{\pm2.6}}$ & $74.9_{\pm0.6}$ & $79.6_{\pm1.0}$ & $63.1_{\pm1.3}$ & $\mathbf{90.5_{\pm1.4}}$ & $62.0_{\pm2.4}$ & $\mathbf{84.0_{\pm1.4}}$ \\
SDSD $K$=1     & $\mathbf{53.4_{\pm1.4}}$ & $56.4_{\pm2.4}$ & $79.5_{\pm0.5}$ & $83.2_{\pm0.8}$ & $76.6_{\pm1.3}$ & $84.1_{\pm1.4}$ & $67.1_{\pm2.7}$ & $78.3_{\pm1.5}$ & $\mathbf{69.1_{\pm0.9}}$ & $75.5_{\pm0.9}$ \\
SDSD $K$=3     & $53.2_{\pm3.3}$ & $56.4_{\pm3.3}$ & $\mathbf{80.8_{\pm1.1}}$ & $84.3_{\pm0.3}$ & $\mathbf{78.6_{\pm1.2}}$ & $\mathbf{85.1_{\pm1.7}}$ & $\mathbf{67.4_{\pm1.8}}$ & $79.5_{\pm2.7}$ & $68.7_{\pm2.7}$ & $75.3_{\pm2.6}$ \\
\bottomrule
\end{tabular}}
\end{table}

\subsection{Diverse External Demonstrations Still Lead to Diversity Collapse}
\label{sec:demos_diversity}

Previously, we evaluated SDSD on demonstrations coming from the student's own correct rollouts. Similar to approaches like OPSD~\citep{opsd}, we now investigate the case when the demonstrations come from external models. 
We will see that self-distillation with external demonstrations still suffers from diversity collapse, regardless of the diversity level of the demonstrations.

We analyze this in the Concept Graph setup. We create multiple demonstration datasets, with different levels of diversity, where each query has multiple correct solutions.
We obtain the demonstrations from different checkpoints of Qwen3-1.7B GRPO+diversity regularizer models. 
For each dataset, we train SDSD models as in \cref{sec:concept-graph} but this time we condition the EMA teacher on these external demonstrations instead of self-generated demonstrations. This way, we can control the diversity of the teacher demonstrations and see their influence on the diversity of the learned student.

All resulting SDSD models have high in-distribution performance (pass@8), similar to the 
models that generated the demonstrations (Fig. \ref{fig:demo_diversity}). Nevertheless, on the Harder Graph setup (that requires diverse exploration during training) pass@8 performance of the students remains low.
Moreover, we compare the semantic diversity of the demonstrations and the diversity of the resulting SDSD models for in-distribution questions. 
We observe that the student models have low semantic diversity, regardless of the level of diversity of the demonstrations. 
This shows that using diverse demonstrations in the teacher does not fix the diversity problem.

\begin{figure}[t]
    \centering
    \makebox[\textwidth]{%
        \hfill
        \begin{subfigure}{0.45\textwidth}
            \centering
            \includegraphics[width=\linewidth]{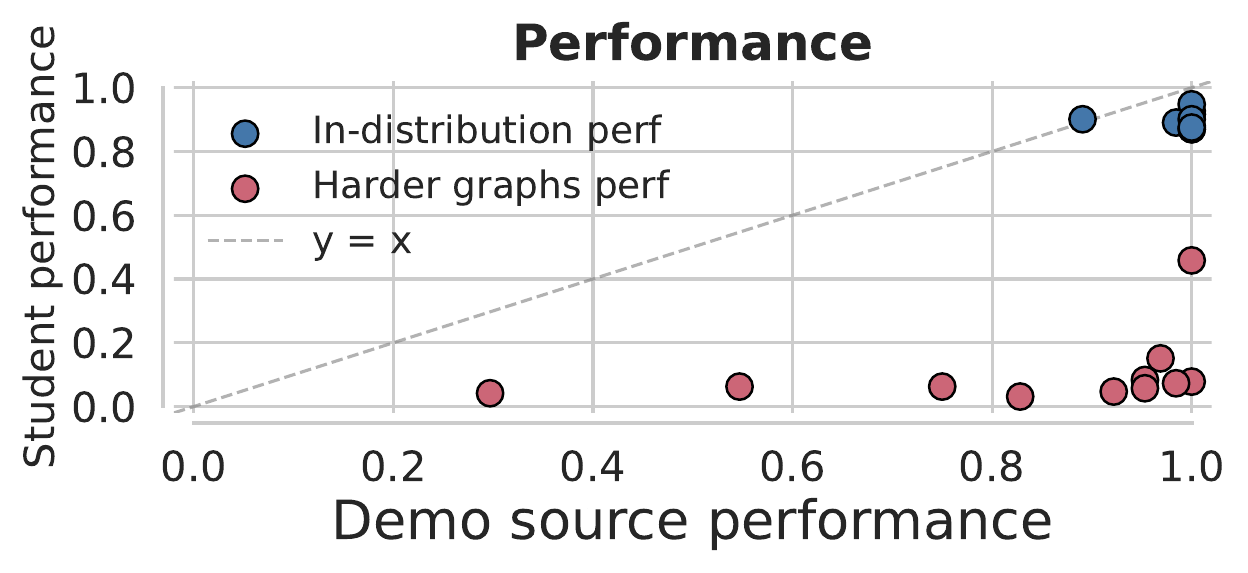}
        \end{subfigure}%
        \hfill
        \begin{subfigure}{0.45\textwidth}
            \centering
            \includegraphics[width=\linewidth]{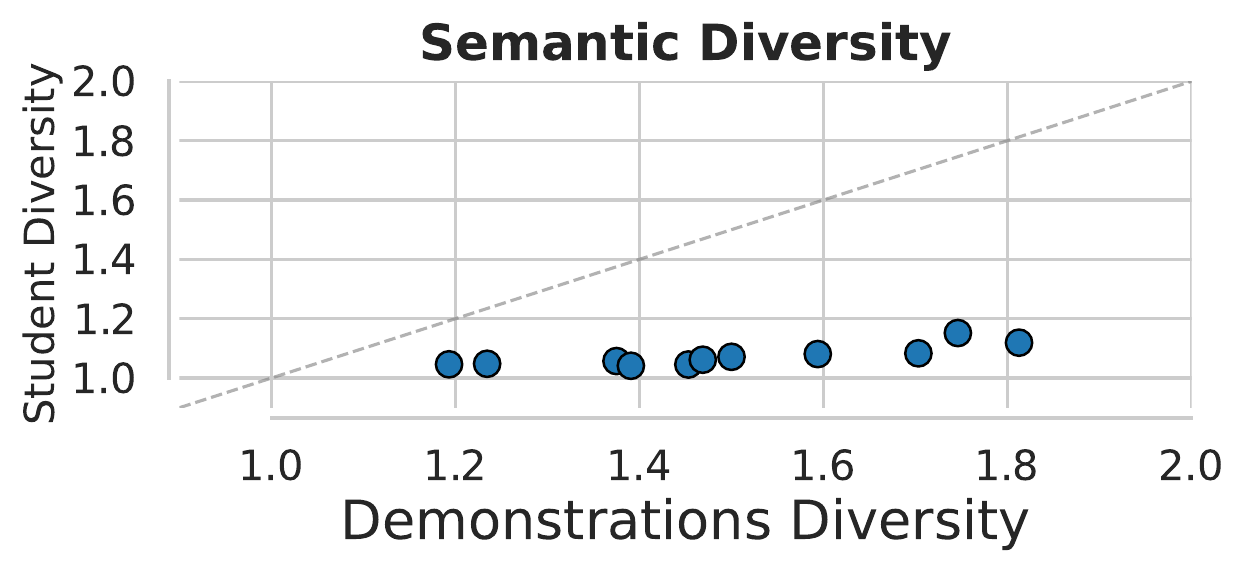}
        \end{subfigure}%
        \hfill
    }%
    \caption{We train SDSD models using external demonstrations that are both correct and diverse. We use multiple datasets of demonstrations with increasing levels of diversity. Each demonstration dataset leads to SDSD models with good in-distribution performance (left plot), but low semantic diversity (right plot), regardless of the level of diversity in the demonstration. This shows that, even with external demonstrations, we still have a problem of diversity collapse, regardless of the level of diversity of the demonstrations. 
    }
    \label{fig:demo_diversity}
\end{figure}

\subsection{Token Entropy Is Not a Sufficient Metric of Diversity}
\label{sec:entropy}

In Fig. \ref{fig:graph_entropy} on Concept Graphs,  SDSD has clearly lower average token-level entropy than the GRPO baselines. This correlates well with SDSD's lack of semantic diversity. On the other hand, token-level entropy cannot distinguish between GRPO and GRPO+diversity, even though GRPO+diversity is clearly more semantically diverse and has better performance on the Harder Graphs task that requires diverse rollouts during training.

Conversely, in the QA setting (Fig. \ref{fig:qa_entropy}), SDSD has \emph{higher} token-level entropy than GRPO despite lower functional diversity and pass@16. 
Again, token-level entropy does not correlate well with a meaningful notion of diversity or performance. This suggests that we need more nuanced notions of diversity than entropy at the token level.

\begin{figure}[!t]
    \centering
    \begin{subfigure}[c]{0.32\textwidth}
        \centering
    \includegraphics[width=1.0\linewidth]{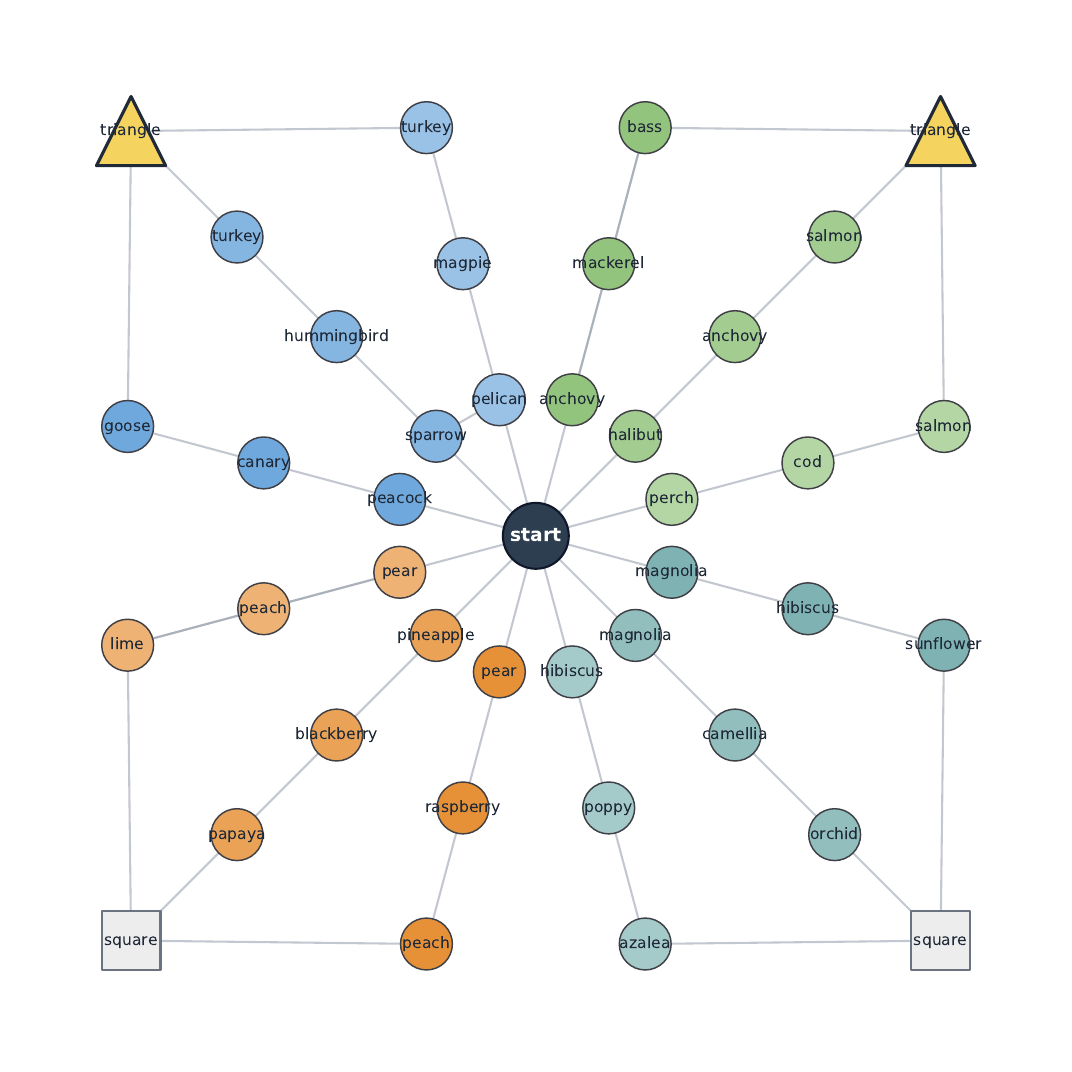}
    \caption{Sample Concept Graph.}
    \label{fig:concept_graph_example}
    \end{subfigure}
    \hfill
    \begin{subfigure}[c]{0.32\textwidth}
    \centering
    \includegraphics[width=1.0\linewidth]{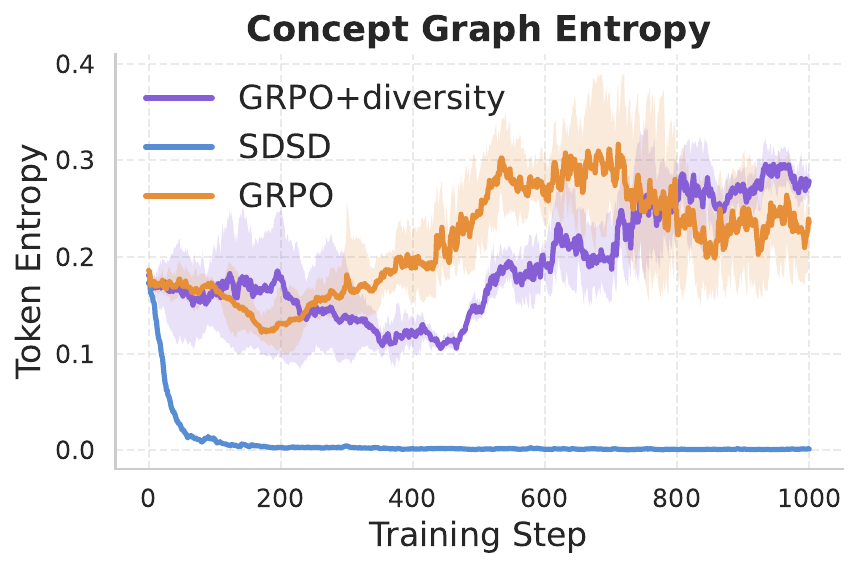}
    \caption{Token Entropy of Concept Graph models}
    \label{fig:graph_entropy}
    \end{subfigure}
    \hfill
    \begin{subfigure}[c]{0.32\textwidth}
        \centering
    \includegraphics[width=1.0\linewidth]{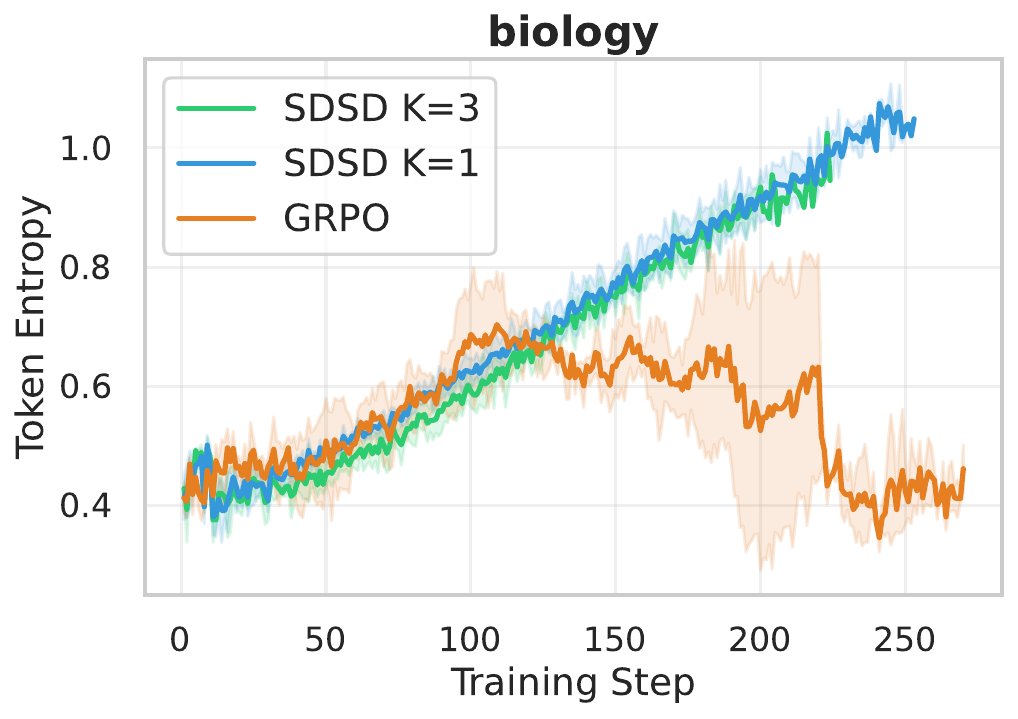}
        \caption{Token Entropy of QA models.}
    \label{fig:qa_entropy}
    \end{subfigure}

    \caption{
    \textbf{(a)} A Concept Graph instance with chain length 3: four concept chains radiate from \emph{start}; the two yellow triangle endpoints are valid targets, the two gray squares are distractors.
    \textbf{(b, c)} Token-level entropy does not tell the whole story.
    For ConceptGraph, token-level entropy cannot explain the higher semantic diversity (Fig. \ref{fig:graph_results}) of GRPO+diversity compared to GRPO. In the QA setting, average token-level entropy does not correlate with functional diversity or low pass@k, since GRPO has the lowest token-level entropy but highest functional diversity and pass@16.}
    \label{fig:entropy}
\end{figure}

\section{Related Work}

\paragraph{On-policy self-distillation.}
Distillation transfers knowledge from a teacher model to a student~\citep{hinton2015distilling, bucilua_2006_original_distillation}, and on-policy distillation gives teacher guidance on student-generated data~\citep{rishab_on_policy_distill}. Moreover, self-distillation methods like SDPO~\citep{sdpo}, OPSD~\citep{opsd, penaloza2026privileged}, SDFT~\citep{sdft_continual}, RLSD~\citep{rlsd}, and OPCD~\citep{opcd} 
use the same model as the teacher to give guidance, by conditioning it on privileged information, achieving strong results. 
As privileged information they use their own correct rollouts, demonstrations from stronger models (a setting that we also use), correct answers, or environment feedback on the student rollouts, such as runtime errors. All these approaches use token-wise supervision to give dense feedback to the student, with \cite{opsd} noting that significant gains come from matching the teacher and student over the whole vocabulary. This dense feedback makes self-distillation achieve high performance in small number of steps \citep{opsd, rlsd} before plateauing or decreasing. SDPO~\citep{sdpo} uses two settings: one where the demonstrations are collected from correct student responses and one where coding runtime feedback is used as privileged information. For the second approach, the alignment between feedback and student rollout should be implicitly higher, since the feedback is generated based on the rollout, thus the teacher should have an easier time understanding their relation and conversely provide good guiding signal. 
RLSD~\citep{rlsd} points to an irreducible gap in the objective of self-distillation, which results in privileged information leakage when the optimization is done, as usual, in mini-batches. They propose to fix this by using the teacher guidance to change the magnitude of the RL gradient, but keep the direction given by the verifiable reward. 

\paragraph{Entropy collapse and mode-seeking.}

Entropy collapse under on-policy training is widely documented~\citep{gx2025kl, beyond_base_2025, invisible_leash}  with prior work attributing it primarily to mode-seeking of the on-policy training. 
Our analysis identifies an additional mechanism specific to self-distillation: unequal alignment between the sampled rollouts and the sampled demonstrations, as defined by PCMI. \citet{nagarajan2025roll_dice} analyze the diversity and creativity of LLMs using controlled graph understanding tasks.

Prior work on RLHF shows that RL-trained models are less diverse than SFT models~\citep{kirk2024_llm_ood_and_diversity}. This is addressed by 
Pass@k-aware training objectives~\citep{chen2025pass_k_learning}, and best-of-N fine-tuning~\citep{chow2025_best_n_inferenceaware} that directly optimize for output coverage. 
\citet{li2025_darling} uses LLM embeddings to partition the rollouts and use them to compute a diversity score. Multiplying the reward with this diversity score improves the diversity of math reasoning and creative writing. 

\section{Conclusion and Limitations}

\paragraph{Conclusion.} We analyzed on-policy self-distillation with sampled demonstrations (SDSD) through the lens of rollout diversity. While SDSD models achieve strong average accuracy, their pass@\(k\) curves are often flat, indicating collapsed functional diversity.
Theoretically, the optimal self-distillation policy tilts the base distribution by a pointwise conditional mutual information (PCMI) score between the student rollout and the demonstration. Unlike standard RL objectives, which preserve probability ratios among equally correct rollouts, this mechanism amplifies pre-existing imbalances and reinforces solutions that already align with the demonstrations and base policy.
Empirically, we observed this diversity collapse on controlled graph path-finding, scientific QA, and synthetic tasks. Our results show that average accuracy alone is insufficient for evaluating self-distillation methods. Functional and semantic diversity are at risk of collapse and should be explicitly monitored during training and deployment.

\paragraph{Scope and limitations.} We focus specifically on the variant of self-distillation that uses sampled correct rollouts as demonstrations. We do not analyze settings where the teacher is conditioned on richer privileged signals, such as runtime errors in coding~\citep{sdpo}, environmental feedback, or external verifiers, where the learning dynamics may differ substantially.
Our theoretical analysis assumes a teacher frozen at the base policy, whereas most practical SDSD implementations, including those used in our experiments, employ an EMA teacher derived from the student itself. In addition, our derivation assumes demonstrations are sampled from the base policy, which more closely resembles OPSD~\citep{opsd} and our external-demonstration graph experiments than the fully self-generated setup of \citet{sdpo}.
In practice, both EMA teachers and self-generated demonstrations introduce additional forms of self-selection bias beyond those captured by our analysis. Finally, our derivation is presented at the sequence level; however, an analogous token-level derivation yields the same PCMI-based tilt at each next-token distribution, causing the effect to compound autoregressively along a trajectory, as discussed in \cref{app:token_level}.

\bibliography{bib}

@misc{NanoAhaMoment2025,
  author       = {Amirhossein Kazemnejad and Milad Aghajohari and Alessandro Sordoni and Aaron Courville and Siva Reddy},
  title        = {Nano Aha! Moment: Single File "RL for LLM" Library},
  year         = {2025},
  howpublished = {\url{https://github.com/McGill-NLP/nano-aha-moment}},
  note         = {GitHub repository}
}

@inproceedings{korbak2022rl_kl,
  title={RL with KL penalties is better viewed as Bayesian inference},
  author={Korbak, Tomasz and Perez, Ethan and Buckley, Christopher},
  booktitle={Findings of the Association for Computational Linguistics: EMNLP 2022},
  pages={1083--1091},
  year={2022}
}

@inproceedings{
beyond_base_2025,
    title={Does Reinforcement Learning Really Incentivize Reasoning Capacity in {LLM}s Beyond the Base Model?},
    author={Yang Yue and Zhiqi Chen and Rui Lu and Andrew Zhao and Zhaokai Wang and Yang Yue and Shiji Song and Gao Huang},
    booktitle={The Thirty-ninth Annual Conference on Neural Information Processing Systems},
    year={2026},
    url={https://openreview.net/forum?id=4OsgYD7em5}
}

@article{sdpo,
  title = {Reinforcement Learning via Self-Distillation},
  author = {Hübotter, Jonas and Lübeck, Frederike and Behric, Lejs and Baumann, Anton and Bagatella, Marco and Marta, Daniel and Hakimi, Ido and Shenfeld, Idan and Kleine Buening, Thomas and Guestrin, Carlos and Krause, Andreas},
  year = {2026},
  journal = {arXiv preprint arXiv:2601.20802},
}

@article{opsd,
  title={Self-Distilled Reasoner: On-Policy Self-Distillation for Large Language Models},
  author={Zhao, Siyan and Xie, Zhihui and Liu, Mengchen and Huang, Jing and Pang, Guan and Chen, Feiyu and Grover, Aditya},
  journal={arXiv preprint arXiv:2601.18734},
  year={2026}
}

@article{sdft_continual,
  title={Self-Distillation Enables Continual Learning},
  author={Shenfeld, Idan and Damani, Mehul and H{\"u}botter, Jonas and Agrawal, Pulkit},
  journal={arXiv preprint arXiv:2601.19897},
  year={2026}
}

@article{opcd,
  title={On-policy context distillation for language models},
  author={Ye, Tianzhu and Dong, Li and Wu, Xun and Huang, Shaohan and Wei, Furu},
  journal={arXiv preprint arXiv:2602.12275},
  year={2026}
}

@article{rlsd,
  title={Self-Distilled RLVR},
  author={Yang, Chenxu and Qin, Chuanyu and Si, Qingyi and Chen, Minghui and Gu, Naibin and Yao, Dingyu and Lin, Zheng and Wang, Weiping and Wang, Jiaqi and Duan, Nan},
  journal={arXiv preprint arXiv:2604.03128},
  year={2026}
}

@article{penaloza2026privileged,
  title={Privileged Information Distillation for Language Models},
  author={Penaloza, Emiliano and Vattikonda, Dheeraj and Gontier, Nicolas and Lacoste, Alexandre and Charlin, Laurent and Caccia, Massimo},
  journal={arXiv preprint arXiv:2602.04942},
  year={2026}
}

@article{sciknoweval,
  title={Sciknoweval: Evaluating multi-level scientific knowledge of large language models},
  author={Feng, Kehua and Shen, Xinyi and Wang, Weijie and Zhuang, Xiang and Tang, Yuqi and Zhang, Qiang and Ding, Keyan},
  journal={arXiv preprint arXiv:2406.09098},
  year={2024}
}

@inproceedings{rishab_on_policy_distill,
  title={On-policy distillation of language models: Learning from self-generated mistakes},
  author={Agarwal, Rishabh and Vieillard, Nino and Zhou, Yongchao and Stanczyk, Piotr and Ramos Garea, Sabela and Geist, Matthieu and Bachem, Olivier},
  booktitle={The twelfth international conference on learning representations},
  year={2024}
}

@article{onpolicydistillation_thinking_machines,
  author = {Lu, Kevin and Thinking~Machines, Lab},
  title = {On-Policy Distillation},
  journal = {Thinking Machines Lab: Connectionism},
  year = {2025},
  note = {https://thinkingmachines.ai/blog/on-policy-distillation},
  doi = {10.64434/tml.20251026},
}

@article{hinton2015distilling,
  title={Distilling the knowledge in a neural network},
  author={Hinton, Geoffrey and Vinyals, Oriol and Dean, Jeff},
  journal={arXiv preprint arXiv:1503.02531},
  year={2015}
}

@inproceedings{bucilua_2006_original_distillation,
  title={Model compression},
  author={Bucilă, Cristian and Caruana, Rich and Niculescu-Mizil, Alexandru},
  booktitle={Proceedings of the 12th ACM SIGKDD international conference on Knowledge discovery and data mining},
  pages={535--541},
  year={2006}
}

@article{invisible_leash,
  title={The invisible leash: Why rlvr may or may not escape its origin},
  author={Wu, Fang and Xuan, Weihao and Lu, Ximing and Liu, Mingjie and Dong, Yi and Harchaoui, Zaid and Choi, Yejin},
  journal={arXiv preprint arXiv:2507.14843},
  year={2025}
}

@inproceedings{
gx2025kl,
title={{KL}-Regularized Reinforcement Learning for Generative Modelling is Designed to Mode Collapse},
author={Anthony GX-Chen and Jatin Prakash and Jeff Guo and Rob Fergus and Rajesh Ranganath},
booktitle={The Fourteenth International Conference on Learning Representations},
year={2026},
url={https://openreview.net/forum?id=flBRtdIihA}
}

@article{li2025_darling,
  title={Jointly reinforcing diversity and quality in language model generations},
  author={Li, Tianjian and Zhang, Yiming and Yu, Ping and Saha, Swarnadeep and Khashabi, Daniel and Weston, Jason and Lanchantin, Jack and Wang, Tianlu},
  journal={arXiv preprint arXiv:2509.02534},
  year={2025}
}

@inproceedings{
kirk2024_llm_ood_and_diversity,
title={Understanding the Effects of {RLHF} on {LLM} Generalisation and Diversity},
author={Robert Kirk and Ishita Mediratta and Christoforos Nalmpantis and Jelena Luketina and Eric Hambro and Edward Grefenstette and Roberta Raileanu},
booktitle={The Twelfth International Conference on Learning Representations},
year={2024},
url={https://openreview.net/forum?id=PXD3FAVHJT}
}

@article{chen2025pass_k_learning,
  title={Pass@ k training for adaptively balancing exploration and exploitation of large reasoning models},
  author={Chen, Zhipeng and Qin, Xiaobo and Wu, Youbin and Ling, Yue and Ye, Qinghao and Zhao, Wayne Xin and Shi, Guang},
  journal={arXiv preprint arXiv:2508.10751},
  year={2025}
}

@inproceedings{
chow2025_best_n_inferenceaware,
title={Inference-Aware Fine-Tuning for Best-of-N Sampling in Large Language Models},
author={Yinlam Chow and Guy Tennenholtz and Izzeddin Gur and Vincent Zhuang and Bo Dai and Aviral Kumar and Rishabh Agarwal and Sridhar Thiagarajan and Craig Boutilier and Aleksandra Faust},
booktitle={The Thirteenth International Conference on Learning Representations},
year={2025},
url={https://openreview.net/forum?id=77gQUdQhE7}
}

@inproceedings{llm_on_policy,
author = {Tajwar, Fahim and Singh, Anikait and Sharma, Archit and Rafailov, Rafael and Schneider, Jeff and Xie, Tengyang and Ermon, Stefano and Finn, Chelsea and Kumar, Aviral},
title = {Preference fine-tuning of LLMs should leverage suboptimal, on-policy data},
year = {2024},
booktitle = {Proceedings of the 41st International Conference on Machine Learning},
}

@article{on_policy_forgetting,
  title={Retaining by doing: The role of on-policy data in mitigating forgetting},
  author={Chen, Howard and Razin, Noam and Narasimhan, Karthik and Chen, Danqi},
  journal={arXiv preprint arXiv:2510.18874},
  year={2025}
}

@article{zhang2025interplay_pre_mid_RL,
  title={On the interplay of pre-training, mid-training, and rl on reasoning language models},
  author={Zhang, Charlie and Neubig, Graham and Yue, Xiang},
  journal={arXiv preprint arXiv:2512.07783},
  year={2025}
}

@inproceedings{
nagarajan2025roll_dice,
title={Roll the dice \& look before you leap: Going beyond the creative limits of next-token prediction},
author={Vaishnavh Nagarajan and Chen Henry Wu and Charles Ding and Aditi Raghunathan},
booktitle={Forty-second International Conference on Machine Learning},
year={2025},
url={https://openreview.net/forum?id=Hi0SyHMmkd}
}

@article{yang2025qwen3,
  title={Qwen3 technical report},
  author={Yang, An and Li, Anfeng and Yang, Baosong and Zhang, Beichen and Hui, Binyuan and Zheng, Bo and Yu, Bowen and Gao, Chang and Huang, Chengen and Lv, Chenxu and others},
  journal={arXiv preprint arXiv:2505.09388},
  year={2025}
}

@misc{olmo3,
      title={OLMo 3},
      author={Team OLMo and Allyson Ettinger and Amanda Bertsch and Bailey Kuehl and David Graham and David Heineman and Dirk Groeneveld and Faeze Brahman and Finbarr Timbers and Hamish Ivison and Jacob Morrison and Jake Poznanski and Kyle Lo and Luca Soldaini and Matt Jordan and Mayee Chen and Michael Noukhovitch and Nathan Lambert and Pete Walsh and Pradeep Dasigi and Robert Berry and Saumya Malik and Saurabh Shah and Scott Geng and Shane Arora and Shashank Gupta and Taira Anderson and Teng Xiao and Tyler Murray and Tyler Romero and Victoria Graf and Akari Asai and Akshita Bhagia and Alexander Wettig and Alisa Liu and Aman Rangapur and Chloe Anastasiades and Costa Huang and Dustin Schwenk and Harsh Trivedi and Ian Magnusson and Jaron Lochner and Jiacheng Liu and Lester James V. Miranda and Maarten Sap and Malia Morgan and Michael Schmitz and Michal Guerquin and Michael Wilson and Regan Huff and Ronan Le Bras and Rui Xin and Rulin Shao and Sam Skjonsberg and Shannon Zejiang Shen and Shuyue Stella Li and Tucker Wilde and Valentina Pyatkin and Will Merrill and Yapei Chang and Yuling Gu and Zhiyuan Zeng and Ashish Sabharwal and Luke Zettlemoyer and Pang Wei Koh and Ali Farhadi and Noah A. Smith and Hannaneh Hajishirzi},
      year={2025},
      eprint={2512.13961},
      archivePrefix={arXiv},
      primaryClass={cs.CL},
      url={https://arxiv.org/abs/2512.13961},
}

@article{shao2024_deepseekmath_grpo,
  title={Deepseekmath: Pushing the limits of mathematical reasoning in open language models},
  author={Shao, Zhihong and Wang, Peiyi and Zhu, Qihao and Xu, Runxin and Song, Junxiao and Bi, Xiao and Zhang, Haowei and Zhang, Mingchuan and Li, YK and Wu, Yang and others},
  journal={arXiv preprint arXiv:2402.03300},
  year={2024}
}

@article{rafailov2023_dpo,
  title={Direct preference optimization: Your language model is secretly a reward model},
  author={Rafailov, Rafael and Sharma, Archit and Mitchell, Eric and Manning, Christopher D and Ermon, Stefano and Finn, Chelsea},
  journal={Advances in Neural Information Processing Systems},
  volume={36},
  pages={53728--53741},
  year={2023}
}

@inproceedings{
adamW,
title={Decoupled Weight Decay Regularization},
author={Ilya Loshchilov and Frank Hutter},
booktitle={International Conference on Learning Representations},
year={2019},
url={https://openreview.net/forum?id=Bkg6RiCqY7},
}

@book{bishop2006pattern,
  title={Pattern recognition and machine learning},
  author={Bishop, Christopher M},
  volume={4},
  number={4},
  year={2006},
  publisher={Springer}
}
\bibliographystyle{apalike}

\newpage
\appendix

\section{Additional Results}
\subsection{Concept-Graph}

Similar to \cref{fig:graph_results}, we show performance and diversity results in \cref{fig:app_graph_results_pass_k} and \cref{fig:app_graph_results_diversity}. 

\begin{figure}[h]
    \centering
        \includegraphics[width=\linewidth]{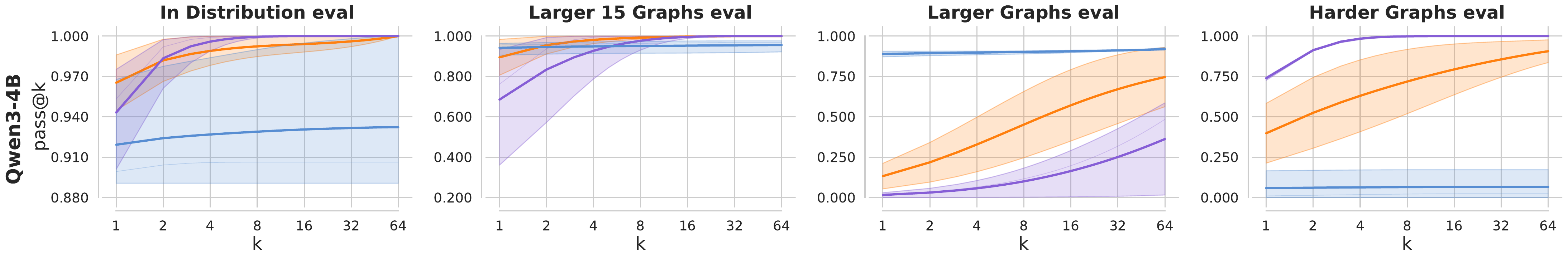}
        \includegraphics[width=\linewidth]{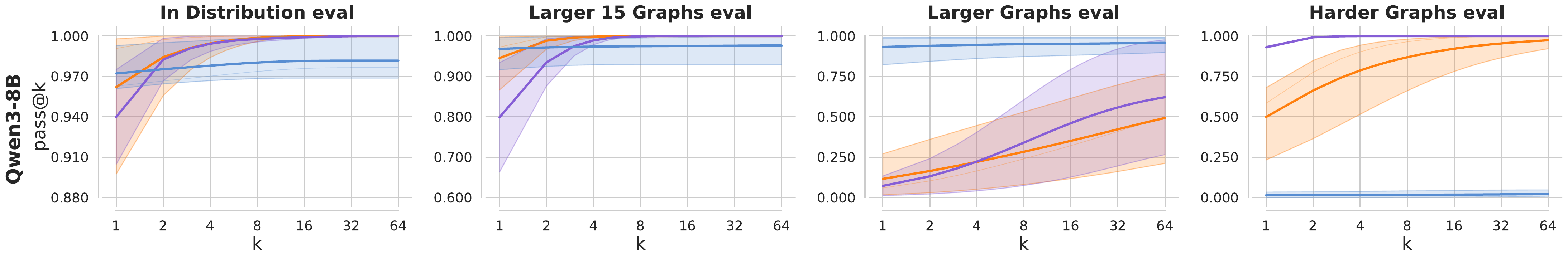}
        \includegraphics[width=\linewidth]{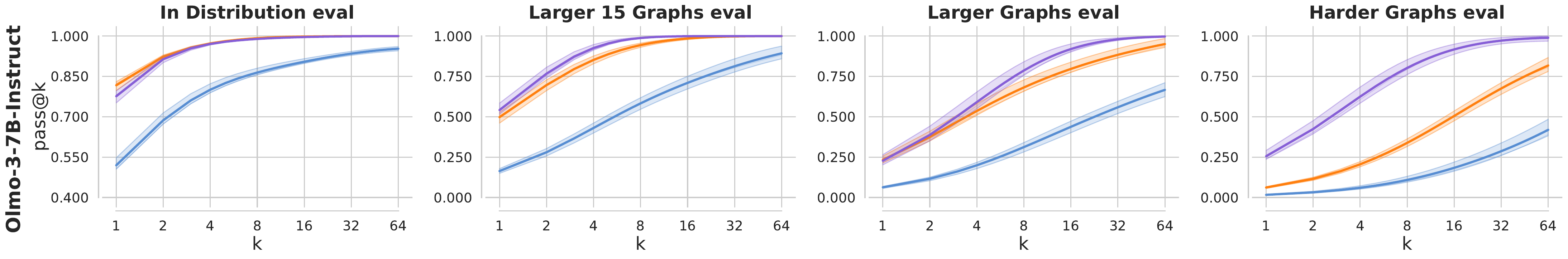}
        \includegraphics[width=0.6\linewidth]{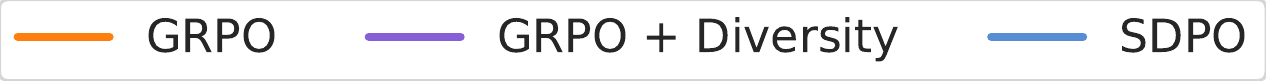}
        
    \caption{
    Compare GRPO, GRPO + Diversity, SDSD in terms of performance (pass@k curves) for models of different sizes. We observe that SDSD usually has more flat pass@k curves for In-Distribution and Larger Graphs (15 nodes or 20 nodes) showing low functional diversity. The Harder Graphs that require learning diverse trajectories during training remains mainly unsolved by SDSD, showing again a lack of output diversity.
    Mean and min/max over 3 seeds are shown.}
    \label{fig:app_graph_results_pass_k}
\end{figure}

\begin{figure}[h]
    \centering
        \includegraphics[width=0.8\linewidth]{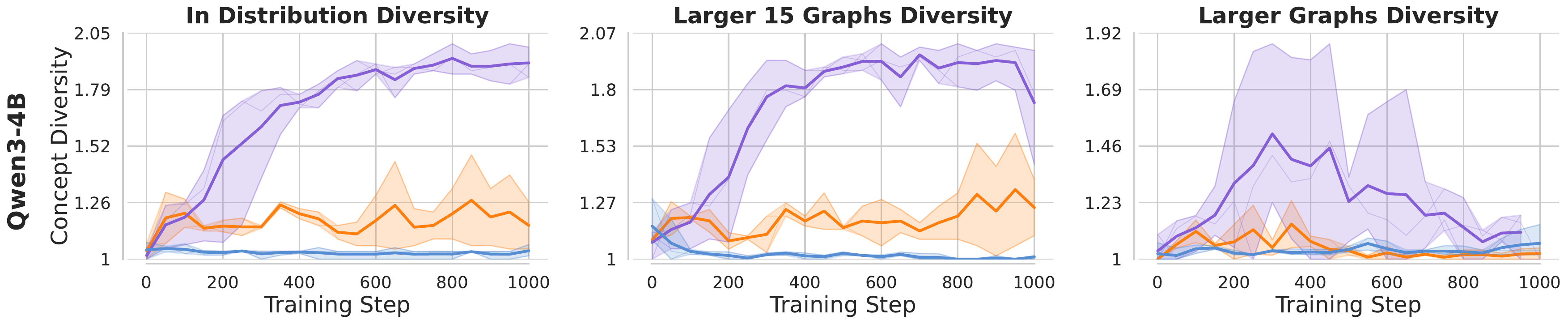}
        \includegraphics[width=0.8\linewidth]{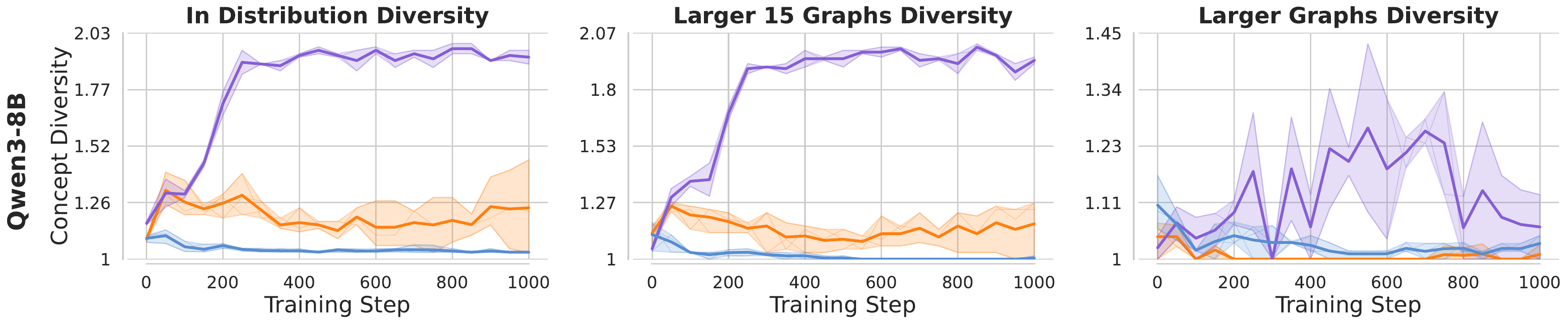}
        \includegraphics[width=0.8\linewidth]{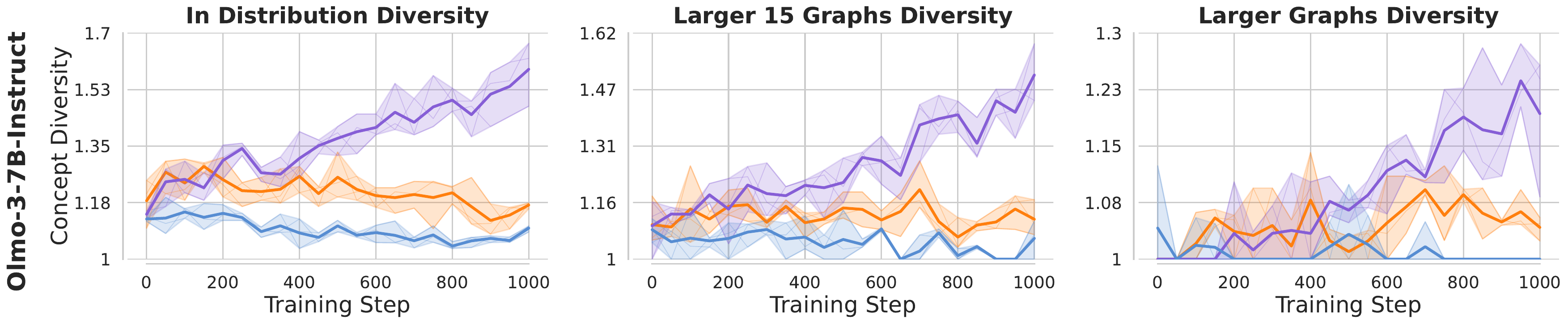}
        \includegraphics[width=0.6\linewidth]{figures/good_figures/main_compare_legend_horizontal.pdf}
    \caption{
    Compare GRPO, GRPO + Diversity, SDSD in terms of semantic diversity for models of different sizes. We observe that SDSD has the lowest concept diversity (as defined in the main text) across training. This explicit measure of diversity gives additional evidence of the loss of diversity in SDSD models. }
    \label{fig:app_graph_results_diversity}
\end{figure}

\subsection{Science QA}

We show the average performance and diversity scores evolution across training steps in Fig. \ref{fig:qa_training_curves}.

\begin{figure}[h]
    \centering
    \includegraphics[width=0.98\linewidth]{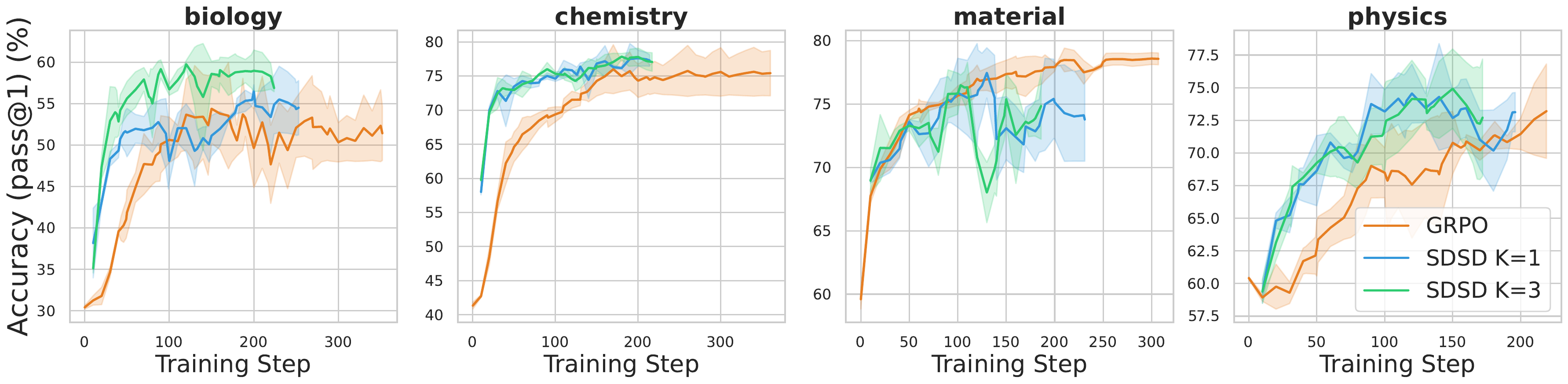}
    \caption{Pass@1 scores of SDSD and GRPO across 5h of training time.}
    \label{fig:qa_training_curves}
\end{figure}

\section{Detailed Derivations and Additional Results}
\label{app:detailed-derivations}

\subsection{Standard KL-Regularized RL}
\label{app:rl-derivation}

We first consider the standard KL-regularized reinforcement learning objective
\begin{equation}
\label{eq:rl-objective}
\max_{\pi} \; \E_{y \sim \pi(\cdot \given x)}\!
\left[
R(y \given x)
\right]
- \beta_{\mathrm{RL}} \, \KL\!\left(\pi(\cdot \given x) \,\|\, \pi_0(\cdot \given x)\right).
\end{equation}

\begin{proposition}[Optimal policy for standard KL-regularized RL]
\label{prop:rl-optimum}
The optimizer of \eqref{eq:rl-objective} is
\begin{equation}
\label{eq:rl-optimum}
\pi^*_{\mathrm{RL}}(y \given x)
\pto
\pi_0(y \given x)
\exp\!\left( \frac{1}{\beta_{\mathrm{RL}}} R(y \given x) \right).
\end{equation}
\end{proposition}

\begin{proof}
Fix $x$ and suppress it in the notation. The objective is
\[
\max_{\pi}
\sum_y \pi(y) R(y)
- \beta_{\mathrm{RL}} \sum_y \pi(y) \log \frac{\pi(y)}{\pi_0(y)}
\quad\text{subject to}\quad
\sum_y \pi(y)=1.
\]
Its Lagrangian is
\[
\mathcal{L}(\pi,\lambda)
=
\sum_y \pi(y) R(y)
- \beta_{\mathrm{RL}} \sum_y \pi(y) \log \frac{\pi(y)}{\pi_0(y)}
+ \lambda \left(\sum_y \pi(y)-1\right).
\]
Differentiating with respect to $\pi(y)$ gives
\[
R(y) - \beta_{\mathrm{RL}}\!\left(\log \frac{\pi(y)}{\pi_0(y)} + 1\right) + \lambda = 0.
\]
Rearranging,
\[
\log \pi(y) = \log \pi_0(y) + \frac{1}{\beta_{\mathrm{RL}}} R(y) + c,
\]
where $c$ is a constant independent of $y$. Exponentiating and normalizing yields
\[
\pi^*_{\mathrm{RL}}(y)
\pto
\pi_0(y) \exp\!\left(\frac{1}{\beta_{\mathrm{RL}}} R(y)\right),
\]
which, after reintroducing the x into the notation, proves \eqref{eq:rl-optimum}.
\end{proof}

\begin{remark}
Equation \eqref{eq:rl-optimum} shows that KL-regularized RL preserves the base policy while exponentially tilting it by the reward.
\end{remark}

\subsection{SDSD-KL: Distillation from a Correct Demonstration}

Consider as teacher the base policy conditioned on a fixed correct demonstration. Let $y^{\mathrm{corr}}$ denote a correct reference demonstration for the same input $x$. We define the teacher by
\begin{equation}
\label{eq:teacher-conditioned}
\pi_0(y \given x, y^{\mathrm{corr}}).
\end{equation}

The corresponding pointwise conditional mutual information (PCMI) is
\begin{equation}
\label{eq:pcmi-fixed}
i(y; y^{\mathrm{corr}} \mid x)
:=
\log \frac{\pi_0(y \given x, y^{\mathrm{corr}})}{\pi_0(y \given x)}.
\end{equation}
This quantity measures how much conditioning on the fixed correct demonstration changes the base policy's log-probability of the candidate sequence $y$. Thus, it can be interpreted as how much support $y^{\mathrm{corr}}$ brings for $y$.

But the demonstration is not fixed, it is sampled from the correct solutions. 
Let
\begin{equation}
\label{eq:pcorr-def}
y^{\mathrm{corr}} \sim p_{\mathrm{corr}}(\cdot \given x)
\end{equation}
be a correct demonstration drawn from a reference distribution over correct solutions. For each realized $y^{\mathrm{corr}}$, we define the teacher
\begin{equation}
\label{eq:q-demo}
q_{y^{\mathrm{corr}}}(y \given x)
:=
\pi_0(y \given x, y^{\mathrm{corr}}).
\end{equation}

The SDSD-KL objective averages the distillation loss over demonstrations:

\begin{equation}
\label{eq:sdkl-objective-pcmi-view}
\min_{\pi} \; \E_{y^{\mathrm{corr}}}
\Big[
\KL\!\left(\pi(\cdot \given x) \,\|\, \pi_0(\cdot \given x, y^{\mathrm{corr}})\right)
\Big]
+ \beta \, \KL\!\left(\pi(\cdot \given x) \,\|\, \pi_0(\cdot \given x)\right).
\end{equation}

\begin{proposition}[Optimal policy for SDSD-KL]
\label{prop:sdkl-optimum}
The optimizer of \eqref{eq:sdkl-objective-pcmi-view} is
\begin{equation}
\label{eq:sdkl-optimum-pcmi}
\pi^*_{\mathrm{SD\text{-}KL}}(y \given x)
\pto
\pi_0(y \given x)
\exp\!\left(
\frac{1}{1+\beta}
\, \E_{y^{\mathrm{corr}} \sim p_{\mathrm{corr}}(\cdot \given x)}
\big[
 i(y; y^{\mathrm{corr}} \mid x)
\big]
\right).
\end{equation}
\end{proposition}

\begin{proof}
Fix $x$ and suppress it in the notation. Expanding \eqref{eq:sdkl-objective-pcmi-view} gives
\[
\min_{\pi}
\E_{y^{\mathrm{corr}}}
\left[
\sum_y \pi(y) \log \frac{\pi(y)}{\pi_0(y \given y^{\mathrm{corr}})}
\right]
+
\beta \sum_y \pi(y) \log \frac{\pi(y)}{\pi_0(y)}
\quad\text{subject to}\quad
\sum_y \pi(y)=1.
\]
Interchanging expectation and summation,
\[
\min_{\pi}
(1+\beta) \sum_y \pi(y) \log \pi(y)
- \sum_y \pi(y) \E_{y^{\mathrm{corr}}}[\log \pi_0(y \given y^{\mathrm{corr}})]
- \beta \sum_y \pi(y) \log \pi_0(y).
\]
The Lagrangian is therefore
\[
\mathcal{L}(\pi,\lambda)
=
(1+\beta) \sum_y \pi(y) \log \pi(y)
- \sum_y \pi(y) \E_{y^{\mathrm{corr}}}[\log \pi_0(y \given y^{\mathrm{corr}})]
- \beta \sum_y \pi(y) \log \pi_0(y)
+ \lambda \left(\sum_y \pi(y)-1\right).
\]
Differentiating with respect to $\pi(y)$ gives
\[
(1+\beta)(\log \pi(y)+1)
- \E_{y^{\mathrm{corr}}}[\log \pi_0(y \given y^{\mathrm{corr}})]
- \beta \log \pi_0(y)
+ \lambda = 0.
\]
Hence,
\[
\log \pi(y)
=
\frac{1}{1+\beta}
\E_{y^{\mathrm{corr}}}[\log \pi_0(y \given y^{\mathrm{corr}})]
+
\frac{\beta}{1+\beta} \log \pi_0(y)
+ c.
\]
Using \eqref{eq:pcmi-fixed},
\[
\log \pi_0(y \given y^{\mathrm{corr}})
=
\log \pi_0(y)
+
 i(y; y^{\mathrm{corr}} ).
\]
Therefore, since taking the expectation over $y^{\mathrm{corr}}$ does not influence the first term: 
\[
\E_{y^{\mathrm{corr}}}[\log \pi_0(y \given y^{\mathrm{corr}})]
=
\log \pi_0(y)
+
\E_{y^{\mathrm{corr}}}
\big[
 i(y; y^{\mathrm{corr}} )
\big].
\]
Substituting back,
\[
\log \pi(y)
=
\log \pi_0(y)
+
\frac{1}{1+\beta}
\E_{y^{\mathrm{corr}}}
\big[
 i(y; y^{\mathrm{corr}})
\big]
+ c.
\]
Exponentiating and normalizing yields:
\begin{equation}
\pi(y)
\pto
\pi_0(y)
\exp\!\left(
\frac{1}{1+\beta}
\, \E_{y^{\mathrm{corr}} }
\big[
 i(y; y^{\mathrm{corr}})
\big]
\right).
\end{equation}
Restoring $x$ to the notation gives \eqref{eq:sdkl-optimum-pcmi}.
\end{proof}

\begin{remark}[Interpretation of SDSD-KL]
\label{rem:sdkl-interpretation}
SDSD-KL has the same formal structure as KL-regularized RL: it exponentially tilts the base policy. The effective reward is now the expected PCMI,
\[
\E_{y^{\mathrm{corr}} \sim p_{\mathrm{corr}}(\cdot \given x)}
\big[
 i(y; y^{\mathrm{corr}} \mid x)
\big],
\]
which measures how strongly, on average over correct demonstrations, the base policy shifts toward the candidate sequence $y$.
\end{remark}

\begin{remark}[Ratio for two correct sequences under SDSD-KL]
\label{rem:sdkl-ratio}
Let $y_1$ and $y_2$ be two correct sequences for the same input $x$, and suppose
\[
\pi_0(y_1 \given x) = k \, \pi_0(y_2 \given x)
\]
for some $k>0$. Then, by \eqref{eq:sdkl-optimum-pcmi}, their ratio under the optimal SDSD-KL policy is
\begin{equation}
\label{eq:sdkl-ratio}
\frac{\pi^*_{\mathrm{SD\text{-}KL}}(y_1 \given x)}{\pi^*_{\mathrm{SD\text{-}KL}}(y_2 \given x)}
=
 k \, \exp\!\left(
\frac{1}{1+\beta}
\E_{y^{\mathrm{corr}} \sim p_{\mathrm{corr}}(\cdot \given x)}
\big[
 i(y_1; y^{\mathrm{corr}} \mid x) - i(y_2; y^{\mathrm{corr}} \mid x)
\big]
\right).
\end{equation}
Equivalently,
\[
\frac{\pi^*_{\mathrm{SD\text{-}KL}}(y_1 \given x)}{\pi^*_{\mathrm{SD\text{-}KL}}(y_2 \given x)}
=
 k \, \exp\!\left(
\frac{1}{1+\beta}
\E_{y^{\mathrm{corr}} \sim p_{\mathrm{corr}}(\cdot \given x)}
\left[
\log\left(
\frac{\exp(i(y_1; y^{\mathrm{corr}} \mid x))}{\exp(i(y_2; y^{\mathrm{corr}} \mid x))}
\right)
\right]
\right).
\]
Thus SDSD-KL preserves the base-policy ratio $k$ only when the two sequences have the same expected PCMI. Otherwise, SDSD-KL further reweights them according to the gap in how strongly correct demonstrations support $y_1$ versus $y_2$.
\end{remark}


\subsection{Optimal Policy for Token-Level SDSD-KL}
\label{app:token_level}
Previously, we discussed the sequence level objective and its implications. This was done for ease of presentation, but in practice, we use a token-level objective. We will see that exactly the same derivations carry in the token-level case, and we will discuss the implications.

Consider the optimization problem at a single generation step $t$. The policy generates the next token $y_t$ from the vocabulary, conditioned on the input $x$ and the generated prefix $y_{<t}$. The token-level SDSD-KL objective averages the distillation loss over demonstrations at the next-token distribution:

\begin{equation}
\min_{\pi} \mathbb{E}_{y^{corr}\sim p_{corr}(\cdot|x)} \left[ \text{KL}(\pi(\cdot|x, y_{<t}) || \pi_0(\cdot|x, y_{<t}, y^{corr})) \right] + \beta \text{KL}(\pi(\cdot|x, y_{<t}) || \pi_0(\cdot|x, y_{<t}))
\end{equation}

\begin{proposition}[Optimal policy for token-level SDSD-KL]
The optimizer of the token-level objective is:
\begin{equation}
\pi_{SD-KL}^*(y_t|x, y_{<t}) \propto \pi_0(y_t|x, y_{<t}) \exp\left(\frac{1}{1+\beta}\mathbb{E}_{y^{corr}\sim p_{corr}(\cdot|x)}[i(y_t; y^{corr}|x, y_{<t})]\right)
\end{equation}
where the token-level pointwise conditional mutual information (PCMI) is defined as:
\begin{equation}
i(y_t; y^{corr}|x, y_{<t}) := \log\frac{\pi_0(y_t|x, y_{<t}, y^{corr})}{\pi_0(y_t|x, y_{<t})}
\end{equation}
\end{proposition}
The proof follows exactly the same steps as before.

\paragraph{Implications of the token-level objective.}
\label{app:token_level_implications}
In the case of sequence-level objective, we have seen that there is a bias for common rollouts. That bias comes from the alignment between rollouts and demonstrations and the preference of the teacher. A similar phenomenon is happening at the token-level. Some tokens are more aligned to the context determined by the demonstration and the previous tokens of the student rollout. Again, next-tokens that will move the current rollout into a novel or uncommon direction might be less aligned to the context. Moreover, given the same context, some next-tokens are preferred over others by the teacher.

On top of this, differently from the sequence level objective, here the teacher is myopic to the full student rollout. It guides the next token without taking into account the relation to future tokens. Thus, some commonalities between student rollout and demonstrations might only be seen at a sequence level and could be harder to establish at intermediate token positions. Thus, the alignment at intermediate token positions might be low, causing learning to be even more biased.
This leads to a bias for common next-tokens, which turns into common entire distributions as we discussed before.

\subsection{Concept Graph Task: Additional details}
\label{app:concept_graph}

\newtcblisting{codeblock}{
  listing only,
  breakable,
  colback=gray!5,
  colframe=gray!40,
  arc=2mm,
  boxrule=0.4pt,
  left=6pt,
  right=6pt,
  top=6pt,
  bottom=6pt,
  listing options={
    basicstyle=\ttfamily\small,
    breaklines=true
  }
}
In the Concept Graph task, each query is a problem of finding a path given to the LLM. Each query has a different graph with randomly sampled structure and node names. Here is an example:

\begin{codeblock}                                   
  You are given the following graph structure.
  Nodes:                                                                                                                                                                                                                                            
    0   start           [hub]                               
    1   triangle        [shape]                                                                                                                                                                                                                     
    2   pigeon          [concept=birds]                                                                                                                                                                                                             
    3   parrot          [concept=birds]                                                                                                                                                                                                             
    4   sparrow         [concept=birds]                                                                                                                                                                                                             
    5   crow            [concept=birds]                     
    6   heron           [concept=birds]
    7   eagle           [concept=birds]
    8   square          [shape]                                                                                                                                                                                                                     
    9   tuna            [concept=fish]                                                                                                                                                                                                           
    ...                                                     
  Edges:
    0 (start) -- 2 (pigeon)
    0 (start) -- 5 (crow)
    0 (start) -- 9 (tuna)  
    0 (start) -- 21 (plum)
    ...                                                                                                            
    2 (pigeon) -- 3 (parrot)                                     
    3 (parrot) -- 4 (sparrow)                               
    1 (triangle) -- 4 (sparrow)                                                              
    ...
    21 (plum) -- 22 (mango)  
    22 (mango) -- 27 (triangle)  
    ...
  Your task is to generate a path from the start node to the target node named triangle. For the output, print the names of the nodes in the path, use the following format: A path from start to triangle is \boxed{start, node_1_name, ..., triangle} for example \boxed{start, Jacobi, Hamilton, ..., star}.        
\end{codeblock}

A correct response following the \textit{birds} concept chain would be:                                                                                                                                                                 
  \begin{quote}
  \small                            
  \verb|A path from start to triangle is \boxed{start, pigeon, parrot, sparrow, triangle}|                                                                                      
  \end{quote}

  \noindent An equally valid alternative following a different concept (\textit{fruits}) would be:                                                                        
  \begin{quote}
  \small                                 
  \verb|\boxed{start, plum, mango, triangle}|
  \end{quote}                                                                                                       
  Both receive the same maximum reward, but a model that \emph{only} produces bird-paths across different graphs exhibits lower solution diversity than one that explores both bird and fruit paths.

\end{document}